\documentclass{article}

\PassOptionsToPackage{numbers,square,sort&compress}{natbib}
\usepackage[preprint]{neurips_2026}

\usepackage[utf8]{inputenc}
\usepackage[T1]{fontenc}
\usepackage{amsmath,amssymb,amsthm}
\usepackage{mathtools}
\usepackage{graphicx}
\usepackage{booktabs}
\usepackage{multirow}
\usepackage{longtable}
\usepackage{algorithm,algorithmic}
\usepackage{xcolor}
\usepackage{listings}
\usepackage{subcaption}
\usepackage{float}
\usepackage{microtype}
\usepackage{placeins}
\usepackage{url}
\usepackage{nicefrac}
\usepackage{titletoc}
\usepackage{hyperref}
\usepackage[capitalize]{cleveref}
\hypersetup{
  colorlinks=true,
  linkcolor=blue,
  citecolor=blue,
  urlcolor=blue
}

\definecolor{runcommandbg}{RGB}{248,249,250}
\definecolor{runcommandframe}{RGB}{210,214,219}
\definecolor{runcommandcomment}{RGB}{93,98,107}
\lstdefinestyle{runcommands}{
  basicstyle=\ttfamily\small,
  backgroundcolor=\color{runcommandbg},
  breaklines=true,
  columns=fullflexible,
  commentstyle=\color{runcommandcomment},
  frame=single,
  framexleftmargin=6pt,
  framexrightmargin=6pt,
  framextopmargin=6pt,
  framexbottommargin=6pt,
  framerule=0.4pt,
  keepspaces=true,
  morecomment=[l]{\#},
  rulecolor=\color{runcommandframe},
  showstringspaces=false,
  xleftmargin=0.02\textwidth,
  xrightmargin=0.02\textwidth
}

\newtheorem{theorem}{Theorem}
\newtheorem{proposition}[theorem]{Proposition}
\newtheorem{lemma}[theorem]{Lemma}
\newtheorem{corollary}[theorem]{Corollary}
\theoremstyle{definition}
\newtheorem{definition}{Definition}
\newtheorem{example}{Example}
\newtheorem{assumption}{Assumption}
\theoremstyle{remark}
\newtheorem{remark}{Remark}

\newcommand{\R}{\mathbb{R}}
\newcommand{\N}{\mathbb{N}}
\newcommand{\Z}{\mathbb{Z}}
\newcommand{\E}{\mathrm{E}}
\newcommand{\Lop}{\mathcal{L}}
\newcommand{\G}{\mathcal{G}}

\newcommand{\bxi}{\boldsymbol{\xi}}
\newcommand{\bv}{\mathbf{v}}
\newcommand{\pr}{\mathrm{pr}}

\title{EqOD: Symmetry-Informed Stability Selection \\ for PDE Identification}

\author{%
  Gnankan Landry Regis N'guessan$^{1,2,3}$ \\
  \And
  Bum Jun Kim$^{4,\ast}$ \\
  \AND
  \begin{minipage}{0.95\textwidth}
  \centering\normalfont\small
  $^{1}$Axiom Research Group \\
  $^{2}$Department of Applied Mathematics and Computational Science,
  NM-AIST, Tanzania \\
  $^{3}$African Institute for Mathematical Sciences (AIMS),
  Research and Innovation Centre, Rwanda \\
  $^{4}$The University of Tokyo, Japan \\
  $^{\ast}$Corresponding author:
  \href{mailto:bumjun.kim@weblab.t.u-tokyo.ac.jp}{\texttt{bumjun.kim@weblab.t.u-tokyo.ac.jp}}
  \end{minipage}
}

\begin{document}

\maketitle

\begin{abstract}
Data-driven identification of partial differential equations (PDEs) relies on sparse regression over a candidate library of differential operators, where larger libraries inflate false positives under observation noise and smaller libraries risk missing true terms. We introduce Equivariant Operator Discovery (EqOD), a fully automatic method combining two library reduction mechanisms. When Galilean invariance is detected from trajectory data via a weak-form structural test, EqOD uses the symmetry-reduced library, eliminating terms that our Galilean exclusion result proves to be absent from the governing equation. Otherwise, it applies randomized LASSO stability selection guided by classical false-positive bounds. A residual-based fallback prevents degradation below the full-library baseline. On 8 PDEs at 4 noise levels, EqOD attains $F_1 = 1.000 \pm 0.000$ on Heat at $20\%$ noise, where WF-LASSO obtains $0.475 \pm 0.181$, official PySINDy 2.0 obtains $0.000$, and the WSINDy reimplementation obtains $0.789$. Under the strict criterion that the mean F1 difference exceeds the larger of the two standard deviations, EqOD wins 7 of 32 cells. WF-LASSO wins none, and the remaining 25 cells are ties. Across all 32 cells, EqOD outperforms PySINDy 2.0.0 in 23 of 32 cells, and all 5 PySINDy wins occur on reaction PDEs. External validation on WeakIdent and PINN-SR datasets gives $F_1 = 1.000$ on all 5 clean benchmarks. NLS, 2D, coupled-system, and cylinder-wake extensions are reported. The Galilean library reduction is proved under explicit autonomy and library assumptions. The stability-selection step is motivated by classical false-positive bounds, while formal guarantees for correlated PDE design matrices remain open.
\end{abstract}

\begin{figure}[t!]
\centering
\includegraphics[width=\textwidth]{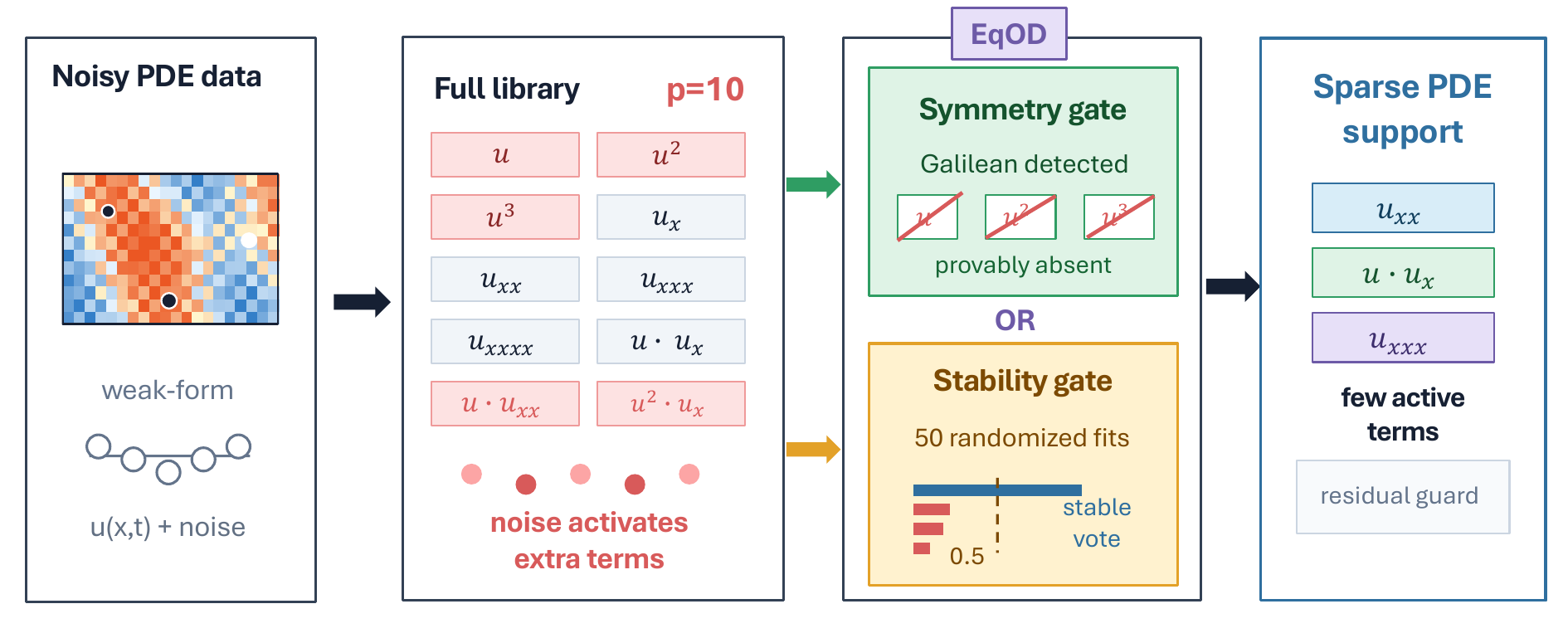}
\caption{High-level EqOD overview. Noisy trajectory data are converted into a weak-form regression system on the full 10-term library. When Galilean invariance is detected, EqOD uses a symmetry gate to remove provably absent reaction terms. Otherwise, it uses a stability gate based on 50 randomized fits to retain consistently selected terms. A residual guard protects the final sparse equation support.}
\label{fig:eqod_overview}
\end{figure}

\section{Introduction}
\label{sec:intro}

\paragraph{The partial differential equation (PDE) identification problem}
Data-driven discovery of PDEs from spatiotemporal trajectory data is a central problem in scientific machine learning. Given noisy measurements $\{u(x_i, t_j) + \epsilon_{ij}\}$ of an unknown physical system on a spatiotemporal grid, the goal is to identify the governing PDE
\begin{equation}\label{eq:pde_general}
u_t = \mathcal{N}[u] = \sum_{j=1}^{s} \xi_j^* \theta_j(u, u_x, u_{xx}, \ldots),
\end{equation}
where $\theta_j$ are the library terms, $\xi_j^*$ are unknown coefficients, and $s$ is the unknown and small number of active terms.

The dominant paradigm, originating with sparse identification of nonlinear dynamics (SINDy) \citep{brunton2016sindy} and later adapted to PDEs through sparse-optimization methods and partial differential equation functional identification of nonlinear dynamics (PDE-FIND) \citep{schaeffer2017learningpde,rudy2017pdefind}, casts PDE identification as sparse regression over a candidate operator library. Specifically, one constructs a candidate library $\Theta = [\theta_1(u), \ldots, \theta_p(u)]$ of $p$ candidate terms and solves
\begin{equation}\label{eq:sindy}
u_t \approx \Theta \bxi, \quad \|\bxi\|_0 \to \min,
\end{equation}
seeking the sparsest coefficient vector $\bxi$ that explains the data. The $\ell_0$ problem is typically relaxed to least absolute shrinkage and selection operator (LASSO) regularization with an $\ell_1$ penalty \citep{tibshirani1996lasso} or solved by sequential thresholded least squares (STLSQ).

\paragraph{The library size problem}
The quality of identification depends critically on the candidate library. Consider the standard polynomial library for PDEs in one spatial and one temporal dimension (1+1D), with up to cubic nonlinearity and fourth-order spatial derivatives:
\begin{equation}\label{eq:library}
\Lop_{\mathrm{std}} = \{u, u^2, u^3, u_x, u_{xx}, u_{xxx}, u_{xxxx}, u \cdot u_x, u \cdot u_{xx}, u^2 \cdot u_x\}.
\end{equation}
This library has $p = 10$ terms. For the Burgers equation, $u_t = -u \cdot u_x + \nu u_{xx}$, only $s = 2$ are active. The remaining $p - s = 8$ are zero-coefficient candidates that observation noise can spuriously activate.

To quantify the false-positive risk, suppose each inactive term has an independent probability $\alpha$ of spurious selection, which is the per-term false-positive rate. The probability of observing at least one false positive is $P(\text{at least one false positive}) = 1 - (1 - \alpha)^{p - s}$.
For $p = 10$, $s = 2$, and $\alpha = 0.05$, we obtain $1 - 0.95^{8} = 0.337$. Reducing the library to $p = 7$ after removing 3 terms via symmetry gives $1 - 0.95^{5} = 0.226$, a 33\% reduction. Further reducing to $p = s$, so that only the true support remains, gives zero probability under this simplified model.
This false-positive calculation leaves the central challenge of reducing the library without knowing the correct answer in advance.

\paragraph{Two complementary reduction mechanisms}
We propose two mechanisms, unified within a single pipeline. The physics-based reduction described in Section~\ref{sec:symmetry_detection} uses Lie point symmetry. When the PDE has Galilean invariance, represented by $u \to u + c$ and $x \to x + ct$, certain library terms are provably absent from any PDE with that symmetry. Detecting this symmetry from trajectory data enables physics-based library pruning with zero false negatives.

The statistics-based reduction in Section~\ref{sec:stability_selection} handles the cases without symmetry. The randomized LASSO stability selection method of \citet{meinshausen2010stability} identifies terms that are consistently selected across random subsamples, and \citet{maddu2022stability} showed that stability-based model selection can improve noisy PDE inference. In favorable regimes, true terms remain consistently selected despite perturbations, whereas noise-correlated terms are selected only under specific conditions. Classical and complementary-pairs stability-selection theory provides finite-sample false-positive bounds under exchangeability or unimodality assumptions \citep{meinshausen2010stability,shah2013stability}. For correlated PDE libraries, EqOD uses these bounds as design guidance and relies on empirical stability profiles for support evidence. The key design principle is to trust the physics when strong structural information is available through detected Galilean symmetry and to trust the statistics otherwise.

A detailed related-work discussion is provided in Appendix~\ref{sec:prior_work}. EqOD's distinguishing gap is the automatic chain from data to symmetry detection, library reduction, and PDE identification.

\subsection{Contributions}

\paragraph{Automatic Lie symmetry detection from data.}
We introduce a weak-form structural test in Algorithm~\ref{alg:galilean} that identifies Galilean invariance from trajectory data with 100\% accuracy, F1 = 1.000, at a threshold $\tau = 0.05$ across 75 test cases generated from 5 PDEs, 3 noise levels, and 5 seeds (Table~\ref{tab:threshold}). The test uses a 6-term targeted regression to disambiguate the convective term $u \cdot u_x$ from reaction terms $u$, $u^2$, and $u^3$.

\paragraph{Provable library reduction.}
The determining equations of the Galilean generator exclude 5 of 10 standard library terms and fix the convective coefficient to $-1$ (Theorem~\ref{thm:reduction}). The implementation conservatively removes 3 terms, while the downstream LASSO handles the remaining 2 exclusions. The worst-case false-positive count drops by 62.5\% for Burgers (Corollary~\ref{cor:fp_symmetry}).

\paragraph{Stability-selected support estimation.}
We use randomized LASSO stability selection as a non-Galilean library-pruning gate inside EqOD, building upon stability-selection theory and prior PDE stability-selection work \citep{meinshausen2010stability,maddu2022stability}. The novelty here is not in stability selection itself but in its role in an automatic mode-selection pipeline that combines symmetry detection, path-dependent library reduction, and residual fallback.

\paragraph{Unified pipeline with automatic mode selection.}
Algorithm~\ref{alg:eqod} automatically selects the physics-based path when Galilean symmetry is detected and the statistics-based path otherwise, with a residual-based fallback to prevent degradation. The branch choice determines the library used in the final weak-form LASSO (WF-LASSO) stage, so mode selection affects the estimated support rather than only the reported method label.

\paragraph{Comprehensive empirical validation across 15 PDE systems.}
The primary benchmark contains 8 scalar one-dimensional (1D) PDEs spanning diffusion, convection, dispersion, reaction, and chaos, with 32 noise conditions, 10 seeds each, one official Python SINDy package (PySINDy) baseline, and two additional baselines. Extensions cover the nonlinear Schr\"odinger equation (NLS) with coupled real and imaginary components, 3 systematic two-dimensional (2D) benchmarks for Heat, advection-diffusion, and Navier--Stokes (NS) vorticity, 2 coupled reaction-diffusion systems, and real cylinder wake data. External validation uses 5 official datasets from the WeakIdent and Physics-Informed Neural Network Sparse Regression (PINN-SR) benchmarks. All extension experiments include WF-LASSO and PySINDy baselines.

The rest of the paper is organized as follows. Section~\ref{sec:method} presents the EqOD pipeline. Section~\ref{sec:experiments} reports the primary 1D benchmark, which contains 8 PDEs, 32 conditions, and 3 baselines. Section~\ref{sec:conclusion} concludes. The appendices provide the technical foundations, proofs, implementation and reproducibility details, and extended empirical evidence, including full benchmark tables, baseline comparisons, ablations, external validations, limitations, failure cases, and supplementary figures (Appendices~\ref{app:notation}--\ref{app:figures}).

\section{Proposed Method: Equivariant Operator Discovery}
\label{sec:method}


EqOD is a four-stage pipeline, shown in Algorithm~\ref{alg:eqod}. It takes trajectory data $\{U^{(m)}\}_{m=1}^{M}$ on a spatial grid $x = (x_1, \ldots, x_{N_x})$ and a temporal grid $t = (t_1, \ldots, t_{N_t})$ and returns the identified PDE with coefficients.

\begin{algorithm}[t!]
\caption{EqOD algorithm}
\label{alg:eqod}
\begin{algorithmic}[1]
\REQUIRE Trajectory data $\{U^{(m)}\}_{m=1}^M$, grids $x$, $t$
\ENSURE Identified PDE: $u_t = \sum_j \hat{\xi}_j \theta_j(u)$
\STATE Stage 1, Lie symmetry detection (Section~\ref{sec:symmetry_detection})
\STATE Run 5 symmetry tests: translation ($\partial_x$, $\partial_t$), Galilean ($t\partial_x + \partial_u$), scaling, reflection
\STATE Record detected generators $\G_{\mathrm{det}} \subseteq \{\bv_1, \ldots, \bv_5\}$
\STATE Stage 2, adaptive library selection (Section~\ref{sec:library_selection})
\IF{Galilean symmetry detected ($t\partial_x + \partial_u \in \G_{\mathrm{det}}$)}
    \STATE $\Lop_{\mathrm{red}} \leftarrow \Lop_{\mathrm{std}} \setminus \{u, u^2, u^3\}$ \quad {\small 7 terms, physics-based}
    \IF{odd reflection is also detected ($-x\partial_x - u\partial_u \in \G_{\mathrm{det}}$)}
        \STATE $\Lop_{\mathrm{red}} \leftarrow \Lop_{\mathrm{red}} \setminus \{u \cdot u_{xx}\}$ {\small 6 terms}
    \ENDIF
    \STATE mode $\leftarrow$ \texttt{symmetry}
\ELSE
    \STATE Build a weak-form system $(\Theta, \mathbf{b})$ on $\Lop_{\mathrm{std}}$, using an odd-reflection-pruned library if odd reflection is detected
    \STATE Run stability selection: $B = 50$ randomized LASSO runs
    \STATE $\hat{\pi}_j \leftarrow$ selection probability for each term $j$
    \STATE $\Lop_{\mathrm{red}} \leftarrow \{j : \hat{\pi}_j > 0.5\}$
    \IF{$|\Lop_{\mathrm{red}}| < 1$}
        \STATE $\Lop_{\mathrm{red}} \leftarrow \Lop_{\mathrm{std}}$ {\small full-library fallback}
    \ENDIF
    \STATE mode $\leftarrow$ \texttt{stability}
\ENDIF
\STATE Stage 3, sparse identification (Appendix~\ref{sec:sparse_id})
\STATE $\hat{\bxi} \leftarrow$ WF-LASSO with 5-fold cross-validation (CV) on $\Lop_{\mathrm{red}}$
\STATE Two rounds of adaptive thresholding + ordinary least squares (OLS) debiasing
\STATE Stage 4, residual-based fallback (Appendix~\ref{sec:fallback})
\STATE $\gamma \leftarrow 1.5$ if mode = \texttt{symmetry}, else $1.2$
\IF{$\|r_{\mathrm{EqOD}}\|^2 > \gamma \cdot \|r_{\text{WF-LASSO}}\|^2$}
    \STATE Use the full-library WF-LASSO result instead
\ENDIF
\RETURN $\hat{\bxi}$, identified equation, mode
\end{algorithmic}
\end{algorithm}

\subsection{Stage 1: Lie symmetry detection from data}
\label{sec:symmetry_detection}

The symmetry detector runs 5 independent tests on the trajectory data. Each test returns a Boolean detection result and a score quantifying the strength of evidence. We summarize the tests here, with the full Galilean detector details in Appendix~\ref{app:galilean_detector_details}.

\paragraph{Galilean detection via weak-form structural test.}
The central detection challenge is Galilean invariance. EqOD tests whether the nonlinear convective term $u \cdot u_x$ contributes significantly to the PDE dynamics by solving a six-term weak-form regression and computing its energy fraction. We set $\tau = 0.05$, which gives F1 = 1.000 on the threshold sweep in Table~\ref{tab:threshold}. The detailed motivation, full procedure, basis choice, and Algorithm~\ref{alg:galilean} are reported in Appendix~\ref{app:galilean_detector_details}. The remaining catalog symmetry tests use lightweight Fourier or parity diagnostics.

\paragraph{Galilean detection accuracy.}
Table~\ref{tab:threshold} reports the detection accuracy as a function of threshold $\tau$.

\begin{table}[t!]
\centering
\caption{Galilean detection: precision, recall, and F1 as functions of threshold $\tau$ across 5 PDEs, 3 noise levels $\{0, 5\%, 10\%\}$, and 5 seeds, for 75 test cases.}
\label{tab:threshold}
\small
\begin{tabular}{@{}lcccccc@{}}
\toprule
$\tau$ & 0.01 & 0.02 & \textbf{0.05} & 0.10 & 0.15 & 0.20 \\
\midrule
Precision & 0.698 & 0.938 & \textbf{1.000} & 1.000 & 1.000 & 1.000 \\
Recall & 1.000 & 1.000 & \textbf{1.000} & 0.867 & 0.700 & 0.500 \\
F1 & 0.822 & 0.968 & \textbf{1.000} & 0.929 & 0.824 & 0.667 \\
\bottomrule
\end{tabular}
\end{table}

The threshold $\tau = 0.05$ achieves perfect classification. At $\tau = 0.01$, there are 13 false positives from reaction PDEs misclassified as Galilean. At $\tau = 0.10$, there are 4 false negatives from Galilean PDEs missed at high noise. The gap between the largest non-Galilean energy fraction, $\approx 0.03$, and the smallest Galilean energy fraction, $\approx 0.08$, provides a robust margin. The threshold was selected on a subset of 5 benchmark PDEs used in the main experiments, so it is an in-sample tuning choice. The wide plateau, where any $\tau \in [0.03, 0.07]$ gives F1 = 1.000, suggests the threshold is not overfit, but held-out validation on PDEs outside the benchmark, such as Sine--Gordon, Benjamin--Bona--Mahony, and Allen--Cahn, would provide stronger evidence. The sweep covers 30 Galilean true-positive cases from Burgers and KdV and 45 non-Galilean true-negative cases from Heat, Fisher-KPP, and Adv-Diff; the remaining benchmark PDEs, KS, KdV--Burgers, and React-Diff, were correctly classified in the main benchmark.

\paragraph{Spatial translation.}
\label{sec:translation_test}
On periodic domains, a spatial shift $u(x, t) \to u(x + a, t)$ only changes the phase of $\hat{u}(k)$, which cancels in the estimated symbol $\hat{\sigma}(k) = \sum_i \overline{\hat{u}(k, t_i)} \hat{u}_t(k, t_i) / \sum_i |\hat{u}(k, t_i)|^2$. We compare $\hat{\sigma}$ before and after spatial shifts of $N_x/8$, $N_x/4$, and $N_x/3$ points. The detection threshold is $\delta_{\mathrm{trans}} < 0.05$.

\paragraph{Temporal translation.}
For autonomous PDEs, the dissipative structure $\mathrm{Re}[\hat{\sigma}(k)]$ is time-independent, while the imaginary part varies for nonlinear PDEs due to mean-flow interaction. We compare $\mathrm{Re}[\hat{\sigma}]$ estimated from the first and second halves of the trajectory, using adaptive windowing to handle strongly dissipative PDEs where signal decays to zero. The detection threshold is $\delta_t < 0.4$.

\paragraph{Scaling.}
\label{sec:scaling_test}
If the PDE has scaling symmetry, $\mathrm{Re}[\hat{\sigma}(k)]$ follows a monomial power law: $\mathrm{Re}[\hat{\sigma}] \propto |k|^\beta$. We perform weighted log-log regression comparing $\log|\mathrm{Re}[\hat{\sigma}]|$ with $\log|k|$ on reliable modes whose power exceeds 5\% of the maximum. The detection criterion is $R^2 > 0.90$.

\paragraph{Reflection.}
\label{sec:reflection_test}
Even reflection is detected by $\|U - U_{\mathrm{flip}}\|^2 / \|U\|^2 < 0.1$, where $u(-x) = u(x)$. Odd reflection is detected by $\|U + U_{\mathrm{flip}}\|^2 / \|U\|^2 < 0.1$, where $u(-x) = -u(x)$. These tests are initial-condition-dependent, so EqOD uses reflection only as an auxiliary signal: odd reflection can further prune the library when paired with stronger structural evidence, but standalone reflection detections are not treated as the main reduction mechanism.

\subsection{Stage 2: Adaptive library selection}
\label{sec:library_selection}

\subsubsection{Physics-based path: symmetry reduction}
\label{sec:physics_path}

When Galilean symmetry is detected, Proposition~\ref{prop:invariant_terms} guarantees that the terms $\{u, u^2, u^3\}$ cannot appear in the PDE because they are not differential invariants and do not appear in any invariant combination beyond $u \cdot u_x$. The reduced library is:
\begin{equation}\label{eq:reduced_library}
\Lop_G = \{u_x, u_{xx}, u_{xxx}, u_{xxxx}, u \cdot u_x, u \cdot u_{xx}, u^2 \cdot u_x\}, \quad |\Lop_G| = 7.
\end{equation}

If odd reflection is additionally detected, EqOD applies a parity prune to the Galilean implementation library. The pure-power term $u^2$ has already been removed by Galilean pruning, so the additional retained-library removal is $u \cdot u_{xx}$, giving $|\Lop_{G,\mathrm{odd}}| = 6$. Figure~\ref{fig:library_reduction} summarizes the path-dependent library sizes and the Heat stability-selection example.

\begin{figure}[t!]
\centering
\includegraphics[width=0.99\textwidth]{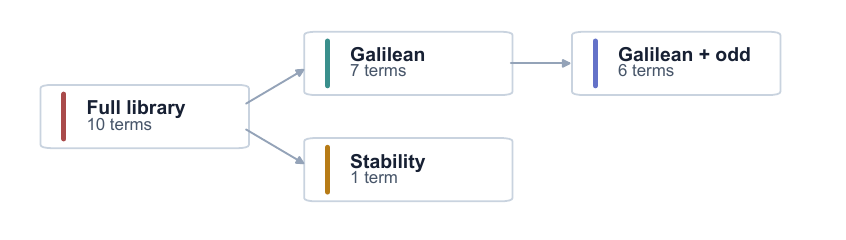}
\caption{Library reduction by EqOD pathway. The standard library has 10 candidate operators. The Galilean path removes the reaction terms $\{u, u^2, u^3\}$, retaining 7 terms. Adding odd reflection removes the remaining parity-incompatible retained term $u \cdot u_{xx}$ and retains 6 terms. In the Heat 10\% stability-selection example, the stable support collapses to $u_{xx}$ only.}
\label{fig:library_reduction}
\end{figure}

\paragraph{Avoiding an extra stability pass after symmetry reduction.}
We tested stability selection on the already-reduced library of 7 terms. It was too aggressive at high noise: on Korteweg--de Vries (KdV) and Kuramoto--Sivashinsky (KS), it sometimes removed true terms, such as $u_{xxx}$ for KdV at 5\% noise, giving F1 = 0.850 instead of 1.000. The physics-based reduction is sufficient because it already removes the reaction-like terms that cause false positives. Adding statistical pruning on top provides no benefit and introduces risk.

\subsubsection{Statistics-based path: stability selection}
\label{sec:stability_selection}

When no Galilean symmetry is detected, we use randomized LASSO stability selection \citep{meinshausen2010stability,buhlmann2011statistics}. The procedure runs $B = 50$ iterations:

\begin{algorithm}[t!]
\caption{Randomized LASSO stability selection}
\label{alg:stability}
\begin{algorithmic}[1]
\REQUIRE Weak-form system $(\Theta, \mathbf{b})$, terms $\Lop$, $B = 50$, $\hat{\pi}_{\mathrm{thr}} = 0.5$
\ENSURE Selection probabilities $\hat{\pi} \in [0, 1]^p$, stable terms $S = \{j : \hat{\pi}_j > 0.5\}$
\STATE Initialize selection counts: $c_j = 0$ for $j = 1, \ldots, p$
\STATE Normalize: $\tilde{\Theta} \leftarrow \Theta / \|\Theta_{\cdot,j}\|$, $\tilde{\mathbf{b}} \leftarrow \mathbf{b} / \|\mathbf{b}\|$
\FOR{$b = 1$ to $B$}
    \STATE Draw subsample: $I_b \subset \{1, \ldots, n\}$, $|I_b| = \lfloor n/2 \rfloor$
    \STATE Draw random penalty weights: $w_j \sim \mathrm{Uniform}(0.5, 1.0)$ for each $j$
    \STATE Form weighted design: $\tilde{\Theta}^{(b)} = \tilde{\Theta}_{I_b} \cdot \mathrm{diag}(w^{-1})$
    \STATE Solve LASSO: $\hat{\bxi}^{(b)} = \arg\min_{\bxi} \lVert\tilde{\mathbf{b}}_{I_b} - \tilde{\Theta}^{(b)} \bxi\rVert_2^2 + \lambda \lVert\bxi\rVert_1$
    \STATE $c_j \leftarrow c_j + \mathbf{1}[|\hat{\xi}_j^{(b)}| > 10^{-6}]$ for each $j$
\ENDFOR
\STATE $\hat{\pi}_j \leftarrow c_j / B$
\RETURN $\hat{\pi}$, $S = \{j : \hat{\pi}_j > 0.5\}$
\end{algorithmic}
\end{algorithm}

\paragraph{Randomized penalty weights.}
Following randomized stability-selection practice \citep{meinshausen2010stability,buhlmann2011statistics}, we use randomized penalty weights $w_j \sim \mathrm{Uniform}(0.5, 1.0)$ rather than uniform penalties. This randomization creates diversity in the LASSO path: true terms with large coefficients are selected regardless of the penalty weight, while noise-correlated terms with small pseudo-coefficients are selected only when their penalty happens to be low. The combined effect of subsampling and random penalties provides a strong stability signal.

\paragraph{Test function density.}
The stability selection path uses a denser grid of test functions, $8 \times 10 = 80$ per trajectory, compared to the Stage 3 pipeline with $5 \times 7 = 35$. This denser grid provides more equations for subsampling and improves the statistical resolution of the selection probabilities.

\subsection{Stages 3 and 4: Sparse identification and fallback}
After library selection, EqOD runs WF-LASSO using the selected library with CV-selected regularization, adaptive thresholding, and OLS debiasing. It then compares the reduced-library residual with the full-library WF-LASSO residual and reverts when the residual ratio exceeds a path-dependent threshold. Full details are in Appendix~\ref{app:sparse_fallback}.

\section{Experiments}
\label{sec:experiments}

\begin{figure}[t!]
\centering
\includegraphics[width=\textwidth]{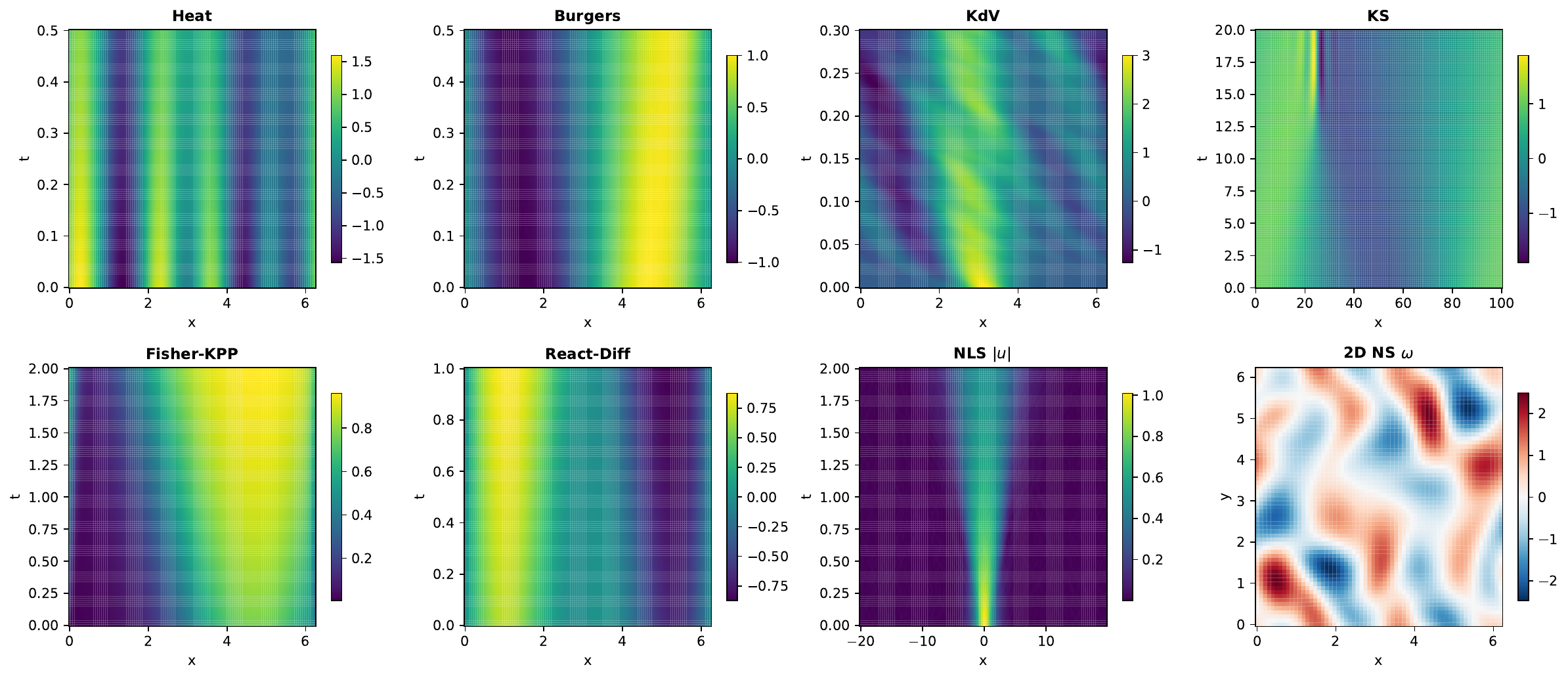}
\caption{Solution snapshots for representative PDEs: the 1D Heat equation, Burgers shock, KdV soliton, KS chaos, Fisher--Kolmogorov--Petrovsky--Piskunov (Fisher-KPP) wavefront, reaction-diffusion, NLS soliton modulus $|u|$, and 2D NS vorticity field at $t_{\mathrm{mid}}$.}
\label{fig:solutions}
\end{figure}

\subsection{Experimental setup}

\paragraph{Benchmark PDEs.}
We evaluate 8 representative 1+1D PDEs spanning convective, dispersive, diffusive, and reaction dynamics. The suite includes classical model equations from turbulence, shallow-water waves, reaction-diffusion fronts, and flame-front chaos, including Burgers \citep{burgers1948mathematical}, KdV \citep{korteweg1895change}, Fisher-KPP \citep{fisher1937wave,kolmogorov1937study}, and KS dynamics \citep{kuramoto1976persistent,sivashinsky1977nonlinear}. We abbreviate advection-diffusion (Adv-Diff) and reaction-diffusion (React-Diff) in tables. Table~\ref{tab:pdes} lists the equations, their true coefficients, Galilean symmetry status, and support size.

\begin{table}[t!]
\centering
\caption{Benchmark PDEs with true equations, Galilean symmetry status, and support size. All PDEs are solved on periodic spatial domains with spectral methods.}
\label{tab:pdes}
\small
\begin{tabular}{@{}llccc@{}}
\toprule
PDE & Equation $u_t = \ldots$ & Galilean & $|S^*|$ & Domain \\
\midrule
Heat & $0.1 u_{xx}$ & No & 1 & $[0, 2\pi]$ \\
Burgers & $-u \cdot u_x + 0.1 u_{xx}$ & Yes & 2 & $[0, 2\pi]$ \\
KdV & $-u \cdot u_x - u_{xxx}$ & Yes & 2 & $[0, 2\pi]$ \\
Fisher-KPP & $0.01 u_{xx} + u - u^2$ & No & 3 & $[0, 2\pi]$ \\
Adv-Diff & $-u_x + 0.05 u_{xx}$ & No & 2 & $[0, 2\pi]$ \\
KS & $-u \cdot u_x - u_{xx} - u_{xxxx}$ & Yes & 3 & $[0, 32\pi]$ \\
KdV--Burgers & $-u \cdot u_x + 0.05 u_{xx} - u_{xxx}$ & Yes & 3 & $[0, 2\pi]$ \\
React-Diff & $0.1 u_{xx} + u - u^3$ & No & 3 & $[0, 2\pi]$ \\
\bottomrule
\end{tabular}
\end{table}

\paragraph{Data and baselines.}
The primary benchmark uses $N_x = 128$, $N_t = 128$, $M = 3$ trajectories per PDE and seed, and noise levels $\sigma_{\mathrm{noise}} \in \{0, 0.05, 0.10, 0.20\}$. Data-generation details, initial conditions, and baseline configurations are reported in Appendices~\ref{sec:data_generation} and~\ref{sec:baseline_details}.

\paragraph{Evaluation metric.}
We use the term-level F1 score to measure identification accuracy:
\begin{equation}
\mathrm{F1} = \frac{2 \cdot \mathrm{Precision} \cdot \mathrm{Recall}}{\mathrm{Precision} + \mathrm{Recall}}, \quad \mathrm{Precision} = \frac{|\hat{S} \cap S^*|}{|\hat{S}|}, \quad \mathrm{Recall} = \frac{|\hat{S} \cap S^*|}{|S^*|},
\end{equation}
where $\hat{S} = \{j : |\hat{\xi}_j| > 10^{-3}\}$ is the identified support and $S^* = \{j : \xi_j^* \neq 0\}$ is the true support.

\subsection{Main results}
\label{sec:main_results}

Table~\ref{tab:main_results} reports the F1 scores for EqOD and WF-LASSO across all 8 PDEs and 4 noise levels, mean $\pm$ std over 10 seeds. The seed flow is documented in Section~\ref{sec:reproducibility}: each seed varies the data trajectories, the noise realization, the bootstrap subsamples in stability selection, and the CV fold permutation in WF-LASSO.

\begin{table}[t!]
\centering
\caption{Main benchmark F1 scores, reported as mean $\pm$ std over 10 seeds, comparing full-library WF-LASSO with EqOD. Bold indicates the better mean. Mode uses S for the symmetry path, SS for the stability selection path, and $\dagger$ for residual fallback to WF-LASSO.}
\label{tab:main_results}
\small
\begin{tabular}{@{}ll cc c c@{}}
\toprule
PDE & Noise & WF-LASSO & EqOD & $\Delta$ & Mode \\
\midrule
\multirow{4}{*}{Heat}
 & 0\%  & $1.000 \pm 0.000$ & $1.000 \pm 0.000$ & $0$ & SS \\
 & 5\%  & $0.733 \pm 0.251$ & $\mathbf{1.000 \pm 0.000}$ & $+0.267$ & SS \\
 & 10\% & $0.554 \pm 0.222$ & $\mathbf{1.000 \pm 0.000}$ & $+0.446$ & SS \\
 & 20\% & $0.475 \pm 0.181$ & $\mathbf{1.000 \pm 0.000}$ & $+0.525$ & SS \\
\midrule
\multirow{4}{*}{Burgers}
 & 0\%  & $1.000 \pm 0.000$ & $1.000 \pm 0.000$ & $0$ & S \\
 & 5\%  & $0.944 \pm 0.176$ & $\mathbf{1.000 \pm 0.000}$ & $+0.056$ & S \\
 & 10\% & $0.700 \pm 0.258$ & $\mathbf{1.000 \pm 0.000}$ & $+0.300$ & S \\
 & 20\% & $0.600 \pm 0.194$ & $\mathbf{0.920 \pm 0.103}$ & $+0.320$ & S \\
\midrule
\multirow{4}{*}{KdV}
 & 0\%  & $1.000 \pm 0.000$ & $1.000 \pm 0.000$ & $0$ & S \\
 & 5\%  & $0.960 \pm 0.084$ & $\mathbf{0.967 \pm 0.105}$ & $+0.007$ & S \\
 & 10\% & $\mathbf{0.960 \pm 0.084}$ & $0.920 \pm 0.103$ & $-0.040$ & S \\
 & 20\% & $\mathbf{0.887 \pm 0.126}$ & $0.787 \pm 0.129$ & $-0.100$ & S \\
\midrule
\multirow{4}{*}{Fisher-KPP}
 & 0\%  & $0.444 \pm 0.000$ & $0.444 \pm 0.000$ & $0$ & SS \\
 & 5\%  & $0.500 \pm 0.000$ & $0.500 \pm 0.000$ & $0$ & SS \\
 & 10\% & $0.494 \pm 0.018$ & $0.494 \pm 0.018$ & $0$ & SS \\
 & 20\% & $0.467 \pm 0.068$ & $0.467 \pm 0.068$ & $0$ & SS \\
\midrule
\multirow{4}{*}{Adv-Diff}
 & 0\%  & $1.000 \pm 0.000$ & $1.000 \pm 0.000$ & $0$ & SS \\
 & 5\%  & $1.000 \pm 0.000$ & $1.000 \pm 0.000$ & $0$ & SS \\
 & 10\% & $1.000 \pm 0.000$ & $1.000 \pm 0.000$ & $0$ & SS \\
 & 20\% & $0.980 \pm 0.063$ & $0.980 \pm 0.063$ & $0$ & SS \\
\midrule
\multirow{4}{*}{KS}
 & 0\%  & $1.000 \pm 0.000$ & $1.000 \pm 0.000$ & $0$ & S \\
 & 5\%  & $0.957 \pm 0.069$ & $0.957 \pm 0.069$ & $0$ & S \\
 & 10\% & $0.914 \pm 0.074$ & $0.914 \pm 0.074$ & $0$ & S \\
 & 20\% & $0.857 \pm 0.080$ & $\mathbf{0.914 \pm 0.074}$ & $+0.057$ & S \\
\midrule
\multirow{4}{*}{KdV--Burgers}
 & 0\%  & $0.779 \pm 0.137$ & $\mathbf{0.957 \pm 0.069}$ & $+0.178$ & S \\
 & 5\%  & $0.464 \pm 0.065$ & $\mathbf{0.613 \pm 0.107}$ & $+0.149$ & S \\
 & 10\% & $0.440 \pm 0.075$ & $0.440 \pm 0.075$ & $0$ & SS$^\dagger$ \\
 & 20\% & $\mathbf{0.427 \pm 0.109}$ & $0.383 \pm 0.173$ & $-0.044$ & SS$^\dagger$ \\
\midrule
\multirow{4}{*}{React-Diff}
 & 0\%  & $0.900 \pm 0.161$ & $0.900 \pm 0.161$ & $0$ & SS \\
 & 5\%  & $1.000 \pm 0.000$ & $1.000 \pm 0.000$ & $0$ & SS \\
 & 10\% & $0.986 \pm 0.045$ & $0.986 \pm 0.045$ & $0$ & SS \\
 & 20\% & $0.914 \pm 0.074$ & $0.904 \pm 0.089$ & $-0.010$ & SS \\
\bottomrule
\end{tabular}
\end{table}

\begin{figure}[t!]
\centering
\includegraphics[width=0.99\textwidth]{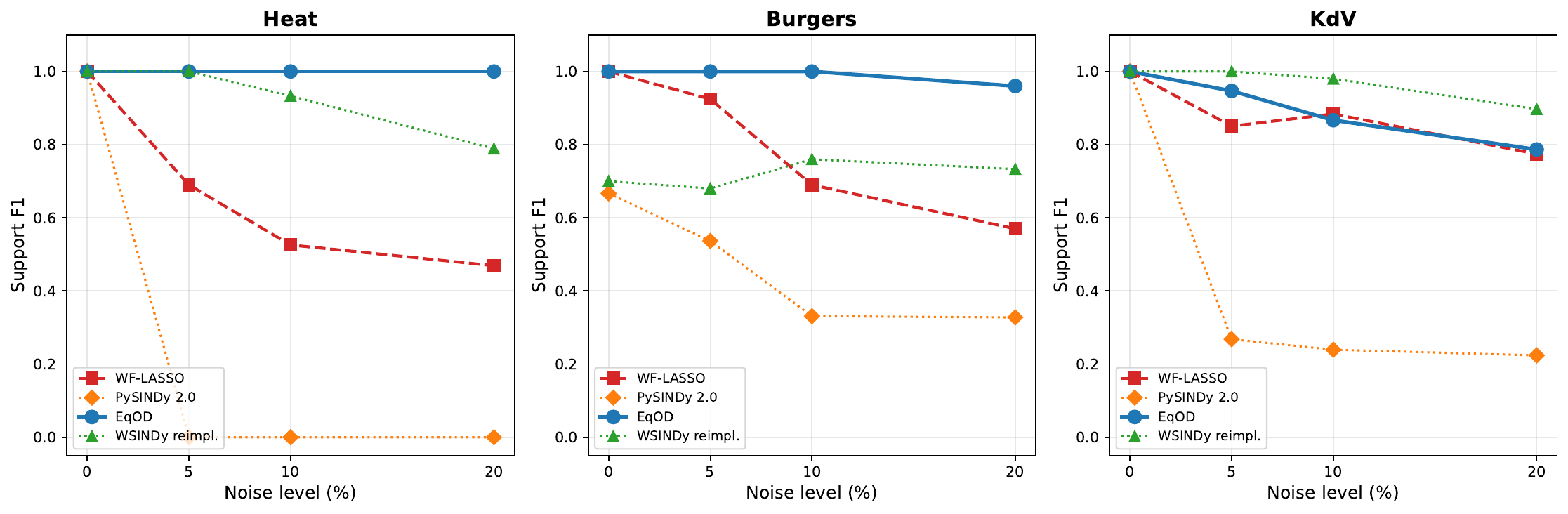}
\caption{F1 scores as a function of noise level for Heat, Burgers, and KdV across 10 seeds.}
\label{fig:heat_burgers}
\end{figure}

The key observations from Table~\ref{tab:main_results} are as follows.
\paragraph{Heat.}
Table~\ref{tab:main_results} and Figure~\ref{fig:heat_burgers} show that on the Heat equation, EqOD achieves F1 $= 1.000 \pm 0.000$ at all noise levels, while WF-LASSO degrades from $1.000$ on clean data to $0.475 \pm 0.181$ at $20\%$ noise. The improvement is due to stability selection isolating $u_{xx}$ as the only stably selected term across all 50 bootstrap subsamples. The standard deviation of $0.000$ for EqOD across the 10 seeds means every bootstrap random-number-generator (RNG) draw and every CV fold permutation produced the same support.

\paragraph{Galilean PDEs.}
For the Galilean PDEs in Table~\ref{tab:main_results}, namely Burgers, KdV, KS, and KdV--Burgers, EqOD uses the symmetry-based path. On Burgers, EqOD wins at $10\%$ and $20\%$ noise. At $0\%$, both methods tie at $F_1 = 1.000$, and at $5\%$, EqOD's mean is higher, 1.000 compared with 0.944, but the gap of 0.056 is smaller than WF-LASSO's standard deviation of 0.176, so the cell is a statistical tie. KS is tied across all four noise levels. The largest gap is at $20\%$, where EqOD's mean of 0.914 exceeds WF-LASSO's 0.857 by 0.057, still below WF-LASSO's standard deviation of 0.080. KdV is also tied across all four noise levels under the patched pipeline. WF-LASSO has a higher mean than EqOD at $10\%$, 0.960 compared with 0.920, and at $20\%$, 0.887 compared with 0.787, but in both cells, the gap is below the larger of the two standard deviations, 0.103 and 0.129, respectively, so neither is a statistically meaningful win for WF-LASSO. We discuss the KdV high-noise regime in Appendix~\ref{sec:limitations}. A directional advantage that survives larger seed counts would still motivate adapting the residual fallback to that regime.

\paragraph{Non-Galilean and reaction PDEs.}
Table~\ref{tab:main_results} also shows that on non-Galilean PDEs without reaction, represented by Adv-Diff, EqOD matches WF-LASSO exactly because stability selection correctly retains the full 2-term support, and no symmetry reduction is applicable. On reaction PDEs, represented by Fisher-KPP and React-Diff, EqOD also matches WF-LASSO. Fisher-KPP is a hard case for all methods because the diffusion coefficient $D = 0.01$ is below the detection threshold given the bump radius and spatial resolution.

\paragraph{Aggregate summary.}
Aggregating Table~\ref{tab:main_results}, counting cells where the mean F1 differs by more than the larger of the two standard deviations, EqOD wins $7$ cells. These are Heat at 5\%, 10\%, and 20\%, Burgers at 10\% and 20\%, and KdV--Burgers at 0\% and 5\%. There are no statistically meaningful WF-LASSO wins, and the remaining $25$ cells are within statistical noise of a tie. The Stage 4 residual fallback is triggered for KdV--Burgers at $10\%$ and $20\%$ noise, which is exactly the regime where the stability path would otherwise fail catastrophically. The strict criterion used here treats both methods as random variables. Under a less conservative criterion using a mean difference greater than half the larger standard deviation, EqOD would additionally win Burgers 5\% and KS 20\%, while retaining the KdV--Burgers 0\% win with additional margin, but we report the conservative count.

The seed-plumbing note for the patched benchmark is reported in Appendix~\ref{sec:seed_plumbing_note}. Clean-data identified equations and coefficient-error details are reported in Appendix~\ref{sec:identified_equations}. Detailed PySINDy and weak-form SINDy (WSINDy) comparisons, ablations and sensitivity analyses, external validation and extensions, and success and failure analysis, including limitations and explicit failure cases, are reported in Appendices~\ref{app:additional_experiments},~\ref{sec:extensions}, and~\ref{sec:analysis}.

\section{Conclusion}
\label{sec:conclusion}

We have introduced EqOD, a fully automatic PDE identification pipeline that combines Lie symmetry detection, stability selection, and residual fallback. On the 1D benchmark with 8 PDEs, 32 noise conditions, and 10 seeds, EqOD records 23/5/4 wins/losses/ties against PySINDy 2.0.0, 16/12/4 against the WSINDy reimplementation, and 10/4/18 against WF-LASSO by mean F1 (Table~\ref{tab:wins}). Under the stricter criterion that the mean gap exceeds the larger of the two standard deviations, EqOD has 7 wins against WF-LASSO, WF-LASSO has none, and 25 cells are ties (Table~\ref{tab:main_results}). On Heat at $20\%$ noise, EqOD reaches F1 $= 1.000 \pm 0.000$, while PySINDy obtains 0.000 and WF-LASSO obtains $0.475 \pm 0.181$ (Tables~\ref{tab:main_results} and~\ref{tab:pysindy}). EqOD also reaches F1 = 1.000 on all 5 external clean-data benchmarks (Table~\ref{tab:external}).

Beyond scalar 1D PDEs, the NLS extension reaches F1 = 1.000 at all noise levels up to 10\%, while WF-LASSO drops to 0.893, and PySINDy fails with an F1 = 0.571 (Table~\ref{tab:nls}). For 2D and coupled systems, EqOD matches WF-LASSO through residual fallback in most cases, and neither method is clearly superior in these regimes (Tables~\ref{tab:2d} and~\ref{tab:coupled}). On real cylinder wake data, the weak-form $R^2$ values are comparable: 0.747 for EqOD and 0.771 for WF-LASSO (Table~\ref{tab:real_data}).

\paragraph{Honest assessment.}
EqOD's gains are concentrated in the setting it was designed for: 1D scalar PDEs with low support ratios. Stability selection does not consistently improve over the full-library baseline on coupled and 2D systems, and the residual fallback can prevent large degradations but cannot create a reliable signal where the reduced library is statistically ambiguous. The Fisher-KPP remains unsolved by all methods because of the small-coefficient blind spot, and WSINDy remains stronger on the KdV at high noise. Detailed limitations and failure cases are provided in Appendix~\ref{sec:limitations}. Future work should tune stability selection for coupled and 2D systems, replacing the 5-test catalog with learned Lie generators \citep{ko2024symmetry}, and extend to nonperiodic domains and nonlocal or fractional operators.

\clearpage
\bibliographystyle{plainnat}
\bibliography{eqod}

\clearpage
\appendix

\startcontents[apx]
\section*{Appendix Table of Contents}
\printcontents[apx]{l}{1}{\setcounter{tocdepth}{2}}

\section{List of Notation}
\label{app:notation}

\begin{table}[t!]
\centering
\caption{Notation used throughout the paper.}
\label{tab:notation}
\begin{tabular}{@{}ll@{}}
\toprule
Symbol & Meaning \\
\midrule
$u(x,t)$ & scalar PDE state \\
$u_t, u_x, u_{(k)}$ & temporal and spatial derivatives \\
$x,t$ & space and time variables \\
$\R,\N,\Z$ & common number sets \\
$N_x,N_t$ & spatial and temporal grid sizes \\
$\Delta x,\Delta t$ & grid spacings \\
$M$ & number of trajectories \\
$U^{(m)}, \tilde{U}$ & clean and noisy trajectory data \\
$\sigma_{\mathrm{noise}}, \sigma$ & noise level \\
$F, \mathcal{N}[u]$ & PDE right-hand side \\
$\theta_j(u)$ & candidate library term \\
$\Theta, \mathbf{b}$ & weak-form design and response \\
$\bxi, \xi_j^*, \hat{\bxi}$ & coefficient vectors \\
$S^*, \hat{S}, s$ & true support, estimate, and size \\
$p$ (library) & number of library terms \\
$\Lop_{\mathrm{std}}$ & standard candidate library \\
$\Lop_{\mathrm{red}}$ & selected reduced library \\
$\Lop_G$ & Galilean-reduced implementation library used by EqOD \\
$\G_{\mathrm{det}}$ & detected symmetry generators \\
$G_\epsilon$ & one-parameter symmetry group \\
$\bv, \pr^{(n)}\bv$ & Lie generator and prolongation \\
$\xi,\tau,\phi$ & generator components \\
$J^{(n)}$ & $n$-th jet space \\
$L[u]$ & linear spatial operator \\
$c_j$ & theoretical library coefficient \\
$\phi_m,\psi$ & test function and bump profile \\
$r_t,r_x$ & test-function radii \\
$B,I_b,w_j$ & bootstrap count, subsample, weights \\
$\hat{\pi}_j,\hat{\pi}_{\mathrm{thr}}$ & selection probability and threshold \\
$\lambda,\lambda^*$ & LASSO regularization values \\
$\eta$ & adaptive coefficient threshold \\
$r_{\mathrm{EqOD}}, r_{\mathrm{std}}, \gamma$ & fallback residuals and threshold \\
$\hat{\sigma}(k), k_n$ & Fourier symbol and wavenumber \\
$f_{uu_x}, \tau$ (test) & Galilean energy fraction and threshold \\
$\mathrm{FP}, \E[\cdot], q$ & false-positive count, expectation, budget \\
$\nu, D$ & diffusion or viscosity coefficients \\
$p,q$ for NLS, $v$ & auxiliary extension fields \\
\bottomrule
\end{tabular}
\end{table}

\section{Related work}
\label{sec:prior_work}

\paragraph{Sparse identification of PDEs}
Automated equation discovery predates SINDy, including symbolic-regression searches for free-form physical laws \citep{schmidt2009distilling}. Earlier spatiotemporal system-identification work also fitted PDE structure or parameters directly from space-time data \citep{voss1998identification,bar1999fitting}. \citet{brunton2016sindy} introduced the SINDy framework for ordinary differential equation discovery, using STLSQ. \citet{schaeffer2017learningpde} and \citet{rudy2017pdefind} extended sparse-regression ideas to PDEs, computing spatial derivatives and applying sparse optimization over candidate libraries. These methods are effective on clean data but degrade under noise because pointwise derivatives amplify high-frequency noise content.

\citet{schaeffer2017sparse} introduced weak-form sparse identification, transferring derivatives from the noisy data to smooth, known test functions. \citet{gurevich2019robust} optimized weak sparse regression for nonlinear PDEs under noise, and \citet{reinbold2020noisy} extended weak-form recovery to noisy, incomplete, and latent-variable spatiotemporal systems. \citet{messenger2021wsindy} developed this idea systematically using compactly supported test functions and STLSQ thresholding.

Several SINDy variants address model selection or robustness from different angles. \citet{mangan2017modelselection} use sparse regression Pareto fronts with information criteria for principled model choice, \citet{kaheman2020sindypi} extend SINDy to implicit and rational dynamics, \citet{champion2019data} combine autoencoders with SINDy to discover coordinates and governing equations jointly, and \citet{fasel2022ensemble} use ensemble bagging to estimate inclusion probabilities under low-data and high-noise regimes. Unified sparse-optimization frameworks broaden the optimizer family available for parsimonious model discovery \citep{champion2020unified}. \citet{desilva2020pysindy} provide the PySINDy library, and the later comprehensive PySINDy release adds PDE, implicit, integral, constrained, and ensemble functionality \citep{kaptanoglu2022pysindy}. PySINDy is the most widely used tool for data-driven PDE identification.

Probabilistic, Bayesian, neural, and evolutionary alternatives include hidden physics models based on Gaussian processes \citep{raissi2018hidden}, threshold sparse Bayesian regression with error bars \citep{zhang2018robust}, PDE-Net and PDE-Net 2.0 for learning differential operators and symbolic response functions from data \citep{long2018pdenet,long2019pdenet2}, physics-informed neural networks for forward and inverse PDE problems \citep{raissi2019pinn}, neural-surrogate PDE discovery in complex datasets or noisy, sparse settings \citep{berg2019complex,xu2021dlpde}, a deep-learning genetic-algorithm PDE method combining neural surrogates with genetic search over incomplete libraries \citep{xu2020dlgapde}, symbolic genetic open-form PDE discovery \citep{chen2022sga}, integral-form deep-learning discovery for sparse and noisy data \citep{xu2021integral}, physics-informed information criteria for noisy and sparse PDE model selection \citep{xu2023pic}, and DeepMoD, which combines neural function approximation with sparse model discovery \citep{both2021deepmod}. These methods address representation, uncertainty, derivative estimation, model selection, or open-library challenges, but they do not automatically detect symmetries and use them to prune the sparse library.

Across this literature, library handling is usually either fixed, open-form search, or model-complexity selection. EqOD addresses the narrower gap of data-adaptive pruning of a sparse PDE library using detected symmetry when available and stability profiles otherwise.

\paragraph{Symmetry-informed identification}
\citet{yang2024equivsindy} constrain the SINDy library to a pre-specified symmetry group, but their EquivSINDy method is restricted to ordinary differential equations and requires the symmetry to be known a priori. Two recent differential-invariant approaches are closest in spirit to our method: symbolic equation discovery with symmetry invariants by \citet{yang2025si} and differential-invariant SINDy by \citet{hu2025di}. However, both assume the symmetry group is known a priori. The symmetry-invariant approach states that the symmetry group must already be supplied and that the current framework cannot be applied when symmetry is unknown. EqOD closes this gap by detecting symmetries from data automatically.

\citet{ko2024symmetry} learn Lie generators from data but do not use them for equation discovery. Their method can discover arbitrary continuous symmetries but is not connected to the PDE identification pipeline. \citet{liegan2023} train generative adversarial networks to discover symmetry generators but again do not apply this to equation discovery. \citet{liu2021aipoincare} and \citet{noethers2024} learn conserved quantities from trajectories. Conservation laws are related to symmetries via Noether's theorem, but the connection to sparse identification of the PDE itself is not made.

\paragraph{Stability selection for regression}
\citet{meinshausen2010stability} introduced stability selection for high-dimensional regression. The method runs LASSO on many random subsamples and retains only terms selected in more than a fraction $\hat{\pi}_{\mathrm{thr}}$ of runs. It provides a finite-sample bound on the expected number of false positives. The approach of \citet{maddu2022stability} applies stability-based model selection to differential-equation learning from limited noisy data, and Ensemble SINDy also uses resampling to produce candidate-term inclusion probabilities \citep{fasel2022ensemble}. EqOD differs by using stability selection only as the non-Galilean branch of an automatic library-pruning stage, paired with a separate symmetry-detection path and a residual fallback.

\paragraph{Gap addressed by EqOD}
No existing method closes the full loop: data $\to$ detect symmetries automatically $\to$ reduce library $\to$ identify PDE. EqOD is the first to do so. The automatic detection capability is not a minor distinction: in practice, the user rarely knows the symmetry group of the PDE they are trying to discover, so automatic detection from data closes this practical gap.

\section{Mathematical Background}
\label{sec:background}

\paragraph{Status note.}
All definitions, lemmas, and theorems in this section are imported from the classical Lie symmetry literature \citep{bluman1989symmetries,olver1993lie, fels1999moving} and are \texttt{verified facts}. We restate them for self-containedness. No new results are claimed until Appendix~\ref{sec:theory}.

\subsection{Lie point symmetries of evolution PDEs}
\label{sec:lie_background}

We consider evolution PDEs of the form
\begin{equation}\label{eq:evolution_pde}
u_t = F(x, t, u, u_x, u_{xx}, \ldots, u_{(n)}),
\end{equation}
where $u = u(x, t)$ is a scalar field on a 1D spatial domain, $u_{(k)} = \partial^k u / \partial x^k$ denotes the $k$-th spatial derivative, and $F$ is a smooth function on the jet space $J^{(n)}$.

\begin{definition}[Lie point symmetry]
A one-parameter Lie group $G_\epsilon$ acting on the space of independent and dependent variables $(x, t, u)$ via
\begin{equation}
(x, t, u) \mapsto (\tilde{x}(x, t, u; \epsilon), \tilde{t}(x, t, u; \epsilon), \tilde{u}(x, t, u; \epsilon))
\end{equation}
is a symmetry of the PDE in Eq.~\ref{eq:evolution_pde} if it maps every solution $u(x,t)$ to another solution $\tilde{u}(\tilde{x}, \tilde{t})$.
\end{definition}

The infinitesimal generator of $G_\epsilon$ is the vector field
\begin{equation}\label{eq:generator}
\bv = \xi(x,t,u) \partial_x + \tau(x,t,u) \partial_t + \phi(x,t,u) \partial_u,
\end{equation}
where $\xi = d\tilde{x}/d\epsilon|_{\epsilon=0}$, $\tau = d\tilde{t}/d\epsilon|_{\epsilon=0}$, and $\phi = d\tilde{u}/d\epsilon|_{\epsilon=0}$.

\begin{definition}[$n$-th prolongation]
The $n$-th prolongation of $\bv$ to the jet space $J^{(n)}$ is
\begin{equation}\label{eq:prolongation}
\pr^{(n)} \bv = \bv + \phi^x \partial_{u_x} + \phi^{xx} \partial_{u_{xx}} + \cdots + \phi^{(n)} \partial_{u_{(n)}},
\end{equation}
where $\phi^{x} = D_x \phi - u_x D_x \xi - u_t D_x \tau$, and higher-order terms follow the recursive prolongation formula \citep{olver1993lie}.
\end{definition}

\begin{theorem}[Lie's infinitesimal criterion {\citep[Theorem 2.31]{olver1993lie}}]\label{thm:lie_criterion}
$G_\epsilon$ is a symmetry of Eq.~\ref{eq:evolution_pde} if and only if
\begin{equation}\label{eq:lie_condition}
\pr^{(n)} \bv \left[ u_t - F \right] = 0 \quad \text{whenever} \quad u_t = F.
\end{equation}
\end{theorem}

For 1+1D evolution PDEs, the relevant point symmetries are spatial translation $\bv_1 = \partial_x$, where the PDE has no explicit dependence on $x$, temporal translation $\bv_2 = \partial_t$, where the PDE is autonomous, Galilean boost $\bv_3 = t \partial_x + \partial_u$, the key symmetry for library reduction, scaling $\bv_4 = a u \partial_u + b x \partial_x + c t \partial_t$, which characterizes monomial operators, and reflection $\bv_5 = -x \partial_x$ for even parity or $\bv_5 = -x \partial_x - u \partial_u$ for odd parity.

\subsection{Galilean invariance: the key structural property}
\label{sec:galilean_theory}

\begin{lemma}[Galilean invariance of convective PDEs]\label{lem:galilean}
Any PDE of the form
\begin{equation}\label{eq:galilean_pde}
u_t + u \cdot u_x = L[u], \quad L[u] = \sum_{k \geq 1} c_k \frac{\partial^k u}{\partial x^k},
\end{equation}
where $L$ is a linear spatial operator with constant coefficients, is invariant under the Galilean boost $(x, t, u) \mapsto (x + ct, t, u + c)$ for all $c \in \R$.
\end{lemma}

\begin{example}[Galilean-invariant PDEs in our benchmark]
The following PDEs have the form in Eq.~\ref{eq:galilean_pde}. Burgers uses $u_t = -u \cdot u_x + \nu u_{xx}$ with $L = \nu \partial_{xx}$. KdV uses $u_t = -u \cdot u_x - u_{xxx}$ with $L = -\partial_{xxx}$. KS uses $u_t = -u \cdot u_x - u_{xx} - u_{xxxx}$ with $L = -\partial_{xx} - \partial_{xxxx}$. KdV--Burgers uses $u_t = -u \cdot u_x + \nu u_{xx} - u_{xxx}$ with $L = \nu\partial_{xx} - \partial_{xxx}$.
\end{example}

\begin{example}[Non-Galilean PDEs]
The following PDEs in our benchmark are not Galilean invariant. Heat, $u_t = \nu u_{xx}$, lacks the $u \cdot u_x$ term. Fisher-KPP, $u_t = D u_{xx} + u - u^2$, has reaction terms $u$ and $u^2$ that break Galilean symmetry. Advection-diffusion, $u_t = -u_x + D u_{xx}$, has linear advection $u_x$, which is not equivalent to $u \cdot u_x$. Reaction-diffusion, $u_t = D u_{xx} + u - u^3$, has reaction terms $u$ and $u^3$ that break Galilean symmetry.
\end{example}

\subsection{Differential invariants and the Cartan--Kuranishi theorem}
\label{sec:invariants_background}

\begin{definition}[Differential invariant]
A smooth function $I: J^{(n)} \to \R$ on the jet space is a differential invariant of the group $G$ if $\pr^{(n)} \bv [I] = 0$ for all generators $\bv$ of $G$.
\end{definition}

The Cartan--Kuranishi theorem \citep{olver1993lie, fels1999moving} guarantees that any $G$-invariant PDE can be written purely in terms of differential invariants. Thus, terms that are not differential invariants and do not appear in invariant combinations cannot appear in a $G$-invariant PDE.

\begin{proposition}[Galilean-invariant terms]\label{prop:invariant_terms}
For the Galilean generator $\bv = t \partial_x + \partial_u$, pure spatial derivatives $u_{(k)} = \partial^k u / \partial x^k$ for $k \geq 1$ are differential invariants because $\pr^{(k)} \bv[u_{(k)}] = 0$. Pure powers of $u$ are not invariants, since $\bv[u] = 1 \neq 0$, $\bv[u^2] = 2u \neq 0$, and $\bv[u^3] = 3u^2 \neq 0$. The product $u \cdot u_x$ is not a differential invariant because $\pr^{(1)}\bv[u \cdot u_x] = u_x \neq 0$, but it appears in Galilean-invariant PDEs as part of the jointly invariant expression $u_t + u \cdot u_x$ (Lemma~\ref{lem:galilean}).
\end{proposition}

\subsection{Weak-form sparse identification}
\label{sec:weak_form_background}

The weak-form approach \citep{schaeffer2017sparse,gurevich2019robust,reinbold2020noisy,messenger2021wsindy} addresses noise sensitivity by replacing pointwise evaluation with integration against test functions. Given a compactly supported test function $\phi \in C_c^\infty(\R^2)$, multiply both sides of $u_t = \sum_j \xi_j \theta_j$ by $\phi$ and integrate:
\begin{equation}\label{eq:weak_form}
\int \int u_t \phi dx dt = \sum_{j=1}^{p} \xi_j \int \int \theta_j(u) \phi dx dt.
\end{equation}
Using the compact support of $\phi$, integration by parts in time gives:
\begin{equation}\label{eq:weak_ibp}
-\int \int u \phi_t dx dt = \sum_{j=1}^{p} \xi_j \int \int \theta_j(u) \phi dx dt.
\end{equation}
The time derivative has been transferred from the noisy data $u$ to the smooth, known test function $\phi_t$. This transfer acts as a natural low-pass filter: integration suppresses high-frequency noise.

\paragraph{Bump test functions.}
We use compactly supported bump functions of the form
\begin{equation}\label{eq:bump}
\psi(r) =
\begin{cases}
\exp\left(\frac{-1}{1 - r^2}\right), & |r| < 1,\\
0, & |r| \geq 1,
\end{cases}
\end{equation}
centered at grid points $(t_c, x_c)$ with radii $(r_t, r_x)$. The 2D test function is the tensor product $\phi_m(t, x) = \psi((t - t_c)/r_t) \cdot \psi((x - x_c)/r_x)$.

Using a grid of $n_t \times n_x$ test function centers yields $N_{\phi} = n_t \cdot n_x$ weak-form equations, indexed by $m = 1,\ldots,N_{\phi}$:
\begin{equation}
b_m = -\int u (\phi_m)_t dx dt, \quad \Theta_{mj} = \int \theta_j(u) \phi_m dx dt.
\end{equation}
The resulting linear system $\mathbf{b} = \Theta \bxi$ is solved via sparse regression.

\paragraph{Spatial derivatives via Fourier spectral method.}
On periodic domains, spatial derivatives are computed via the Fourier pseudospectral method:
\begin{equation}
\widehat{u_{(k)}}(n) = (ik_n)^k \hat{u}(n), \quad k_n = \frac{2\pi n}{L},
\end{equation}
where $\hat{u}$ is the discrete Fourier transform and $L$ is the domain length. The Fourier pseudospectral method is spectrally accurate for smooth solutions and avoids the noise amplification of finite-difference derivatives.

\section{Theoretical Analysis}
\label{sec:theory}

\subsection{Determining equations and library reduction}

The central theoretical result is that Galilean symmetry provably excludes specific library terms. We derive the exclusion result from the determining equations of the symmetry, not merely from the differential invariant criterion.

\begin{assumption}[Autonomy]\label{ass:autonomy}
The right-hand side $F$ has no explicit dependence on $x$ or $t$: $F = F(u, u_x, u_{xx}, u_{xxx}, u_{xxxx})$. This autonomy property holds for every term in $\Lop_{\mathrm{std}}$. Additionally, the functions $\{1, u, u^2, u_x, u_{xx}, u \cdot u_x\}$ are linearly independent on the jet space $J^{(2)}$.
\end{assumption}

\begin{theorem}[Galilean exclusion]\label{thm:reduction}
Under Assumption~\ref{ass:autonomy}, let $u_t = F(u, u_x, u_{xx}, u_{xxx}, u_{xxxx})$ be a 1+1D evolution PDE admitting the Galilean generator $\bv = t \partial_x + \partial_u$. Write $F$ as a linear combination of the standard library in Eq.~\ref{eq:library}:
\begin{equation}
F = c_1 u + c_2 u^2 + c_3 u^3 + c_4 u_x + c_5 u_{xx} + c_6 u_{xxx} + c_7 u_{xxxx} + c_8 u \cdot u_x + c_9 u \cdot u_{xx} + c_{10} u^2 \cdot u_x.
\end{equation}
Then $c_1 = c_2 = c_3 = c_9 = c_{10} = 0$ and $c_8 = -1$. The Galilean-admissible sublibrary is:
\begin{equation}
\Lop_G^{\mathrm{theory}} = \{u_x, u_{xx}, u_{xxx}, u_{xxxx}, u \cdot u_x\}, \quad |\Lop_G^{\mathrm{theory}}| = 5,
\end{equation}
with the convective term $u \cdot u_x$ having fixed coefficient $-1$. The effective number of free parameters is 4.
\end{theorem}

\begin{remark}[Implementation note]
EqOD uses a conservative 7-term implementation library $\Lop_G^{\mathrm{impl}} = \Lop_{\mathrm{std}} \setminus \{u, u^2, u^3\}$ that does not enforce $c_8 = -1$ or exclude $u \cdot u_{xx}$ and $u^2 \cdot u_x$. The LASSO in Stage 3 handles these: the fixed $c_8$ and the zero $c_9, c_{10}$ emerge from the sparse regression. This conservative choice avoids relying on exact coefficient constraints that could be violated by numerical errors or model misspecification.
\end{remark}

\begin{corollary}[Optimality of Galilean reduction]\label{cor:optimality}
The theoretical Galilean-reduced library $\Lop_G^{\mathrm{theory}}$ is the unique maximal sublibrary of $\Lop_{\mathrm{std}}$ containing all terms that can appear with a free coefficient, or with the fixed Galilean convective coefficient, in some Galilean-invariant PDE. No term in $\Lop_{\mathrm{std}} \setminus \Lop_G^{\mathrm{theory}}$ can appear with a nonzero coefficient in any Galilean-invariant PDE in the standard library family.
\end{corollary}

\subsection{Deterministic false-positive reduction}

\begin{corollary}[Deterministic false-positive reduction]\label{cor:fp_symmetry}
Let $\hat{S}$ be the support identified by any sparse regression procedure on library $\Lop$, and let $S^*$ be the true support with $|S^*| = s$. Define the false-positive count $\mathrm{FP}(\Lop) = |\hat{S} \setminus S^*|$. If Theorem~\ref{thm:reduction} removes $k$ terms from $\Lop$ that are provably not in $S^*$, then:
\begin{equation}
\mathrm{FP}(\Lop_G^{\mathrm{theory}}) \leq \mathrm{FP}(\Lop_{\mathrm{std}}) \quad \text{and} \quad \max_{\hat{S}} \mathrm{FP}(\Lop_G^{\mathrm{theory}}) = |\Lop_G^{\mathrm{theory}}| - s = (|\Lop_{\mathrm{std}}| - k) - s.
\end{equation}
For the Galilean case with the full determining equation result, $k = 5$ because $\{u, u^2, u^3, u \cdot u_{xx}, u^2 \cdot u_x\}$ are excluded. The maximum possible false-positive count drops from $|\Lop_{\mathrm{std}}| - s$ to $|\Lop_{\mathrm{std}}| - k - s$. For Burgers, where $s = 2$, this drops from 8 to 3, a 62.5\% reduction in the worst-case false-positive count.
\end{corollary}

\begin{remark}[Probabilistic interpretation]
Under any model, including LASSO, STLSQ, and stability selection under mild regularity, where false-positive probability is monotone in the number of inactive candidates, reducing the candidate count from $p - s$ to $p - s - k$ reduces the false-positive probability. The exact reduction factor depends on the correlation structure and regularization path and is not captured by a simple binomial model. Corollary~\ref{cor:fp_symmetry} avoids this issue by providing a deterministic worst-case guarantee.
\end{remark}

\subsection{Stability selection bound and scope}

\begin{remark}[Classical stability-selection false-positive bound]\label{rem:fp_stability}
Under the standard stability-selection assumptions of \citet[Theorem 1]{meinshausen2010stability}, Algorithm~\ref{alg:stability} with a threshold $\hat{\pi}_{\mathrm{thr}} \in (0.5, 1]$ and a regularization level $\lambda$ such that the expected number of selected terms satisfies $\E[|S_\lambda|] \leq q$ obeys:
\begin{equation}\label{eq:fp_bound}
\E[\mathrm{FP}] \leq \frac{q^2}{(2\hat{\pi}_{\mathrm{thr}} - 1) \cdot p},
\end{equation}
where $p = |\Lop|$ is the library size.

EqOD uses this classical bound as motivation rather than as a formal guarantee for the weak-form PDE design matrices used here. The assumptions in \citet{meinshausen2010stability} require exchangeability of the inactive feature columns, while the complementary-pairs refinement of \citet{shah2013stability} gives bounds under shape restrictions such as unimodality. In our weak-form PDE setting, columns such as $u_{xx}$ and $u_{xxxx}$ are structurally correlated through the spectral truncation, and LASSO support recovery is known to depend on covariance conditions such as irrepresentability and sample-size thresholds \citep{zhao2006model,wainwright2009sharp}. Elastic-net penalties can be preferable for grouped correlated predictors \citep{zou2005elasticnet}, but EqOD uses LASSO because its sparsity path and stability-selection bounds are simpler to interpret. The empirical selection profiles in Table~\ref{tab:stability_profiles} provide the actual evidence of separation between true and false terms. A rigorous bound for correlated PDE design matrices is an open problem.
\end{remark}

\begin{remark}[Practical bound evaluation]
Empirically, true PDE terms achieve selection probabilities $\hat{\pi} > 0.9$ while false terms have $\hat{\pi} < 0.3$ (Table~\ref{tab:stability_profiles}). Evaluating the bound at $\hat{\pi}_{\mathrm{thr}} = 0.9$, which all true terms exceed, with $q = 3$ gives:
\begin{equation}
\E[\mathrm{FP}] \leq \frac{9}{(2 \cdot 0.9 - 1) \cdot 10} = \frac{9}{8} = 1.125.
\end{equation}
Under the classical assumptions, the bound says that on average, at most 1.1 false positives are expected when using the empirically observed threshold. The operational threshold $\hat{\pi}_{\mathrm{thr}} = 0.5$ used in Algorithm~\ref{alg:stability} is conservative: it retains all terms with majority support, including those near the boundary. The adaptive thresholding in Stage 3 then removes terms with small coefficients.
\end{remark}

\subsection{Consistency of the Galilean detector}

\begin{remark}[Heuristic consistency argument for the Galilean test]\label{rem:detection}
Consider the weak-form Galilean test in Algorithm~\ref{alg:galilean} applied to trajectories of a PDE $u_t = F(u, u_x, \ldots)$ on a periodic domain with grid spacing $\Delta x = L/N_x$ and $\Delta t = T/N_t$. Let $f_{uu_x}^{(N)}$ denote the energy fraction computed on the $N_x \times N_t$ grid. The expected asymptotic behavior has three regimes. If $F$ contains $u \cdot u_x$ with a nonzero coefficient in a Galilean PDE, then $f_{uu_x}^{(N)} \to f_{uu_x}^* > 0$ as $N_x, N_t \to \infty$. If $F$ does not contain $u \cdot u_x$ in a non-Galilean PDE, then $f_{uu_x}^{(N)} \to 0$ as $N_x, N_t \to \infty$. Under additive noise $\tilde{U} = U + \sigma \cdot \mathrm{std}(U) \cdot Z$ with independent standard normal entries $Z_{ij}$, these convergence statements are expected to hold for fixed $\sigma < \sigma_{\max}$, where $\sigma_{\max}$ depends on the signal-to-noise ratio in the weak-form integrals.

The Galilean and non-Galilean asymptotic claims are supported by the spectral accuracy of trapezoidal quadrature on periodic domains. The noisy-data claim is heuristic. The procedure is hybrid, using spectral derivatives of noisy data followed by weak averaging against bump test functions, rather than a pure weak form that transfers all derivatives to the test function via integration by parts. Spectral derivatives amplify high-frequency noise, with $\mathrm{Var}[\widehat{u_{(k)}}] \propto k^{2k} \sigma^2$, but the subsequent integration against compactly supported $\phi$ suppresses this contribution. The $O(1/\sqrt{N_x N_t})$ rate stated above is a heuristic based on the law of large numbers. A rigorous central limit theorem for spectrally differentiated noisy data integrated against bump functions is an open technical question. Adopting full weak-form integration by parts for the library columns, as in WSINDy \citep{messenger2021wsindy}, would yield a cleaner convergence argument and is a natural improvement.

Empirically, the test achieves F1 = 1.000 across 75 test cases generated from 5 PDEs, 3 noise levels, and 5 seeds at threshold $\tau = 0.05$. This result does not constitute a proof of consistency. It provides evidence that the heuristic works on the tested PDEs.
\end{remark}

\subsection{Completeness of the physics-based path}

\begin{corollary}[Completeness]\label{cor:completeness}
If the Galilean detector correctly identifies a PDE as Galilean-invariant, the implementation-reduced library $\Lop_G$ contains all true terms: $S^* \subseteq \Lop_G$.
\end{corollary}

\section{Proofs of stated results}
\label{app:proofs}

\subsection{Proof of Lemma~\ref{lem:galilean}}

\begin{proof}
Define $\tilde{u}(\tilde{x}, \tilde{t}) = u(\tilde{x} - c\tilde{t}, \tilde{t}) + c$. Compute the partial derivatives:
\begin{align}
\tilde{u}_{\tilde{t}} &= u_t(\tilde{x} - c\tilde{t}, \tilde{t}) - c u_x(\tilde{x} - c\tilde{t}, \tilde{t}), \\
\tilde{u}_{\tilde{x}} &= u_x(\tilde{x} - c\tilde{t}, \tilde{t}), \\
\tilde{u}_{\tilde{x}\tilde{x}} &= u_{xx}(\tilde{x} - c\tilde{t}, \tilde{t}), \quad \text{etc.}
\end{align}
The convective term transforms as:
\begin{equation}
\tilde{u} \cdot \tilde{u}_{\tilde{x}} = (u + c) \cdot u_x = u \cdot u_x + c u_x.
\end{equation}
The linear operator transforms as $L[\tilde{u}] = L[u + c] = L[u]$ since $L$ has constant coefficients and $c$ is constant. Substituting:
\begin{equation}
\tilde{u}_{\tilde{t}} + \tilde{u} \cdot \tilde{u}_{\tilde{x}} = (u_t - c u_x) + (u \cdot u_x + c u_x) = u_t + u \cdot u_x = L[u] = L[\tilde{u}].
\end{equation}
The transformed function satisfies the same PDE.
\end{proof}

\subsection{Proof of Proposition~\ref{prop:invariant_terms}}

\begin{proof}
The prolonged generator has $\phi^{x} = D_x(1) - u_x \cdot D_x(t) - u_t \cdot D_x(0) = 0$, and similarly $\phi^{(k)} = 0$ for all $k \geq 1$ by the recursive prolongation formula.
For pure powers, $\bv[u] = \phi = 1$. By the chain rule, $\bv[u^n] = n u^{n-1} \cdot \bv[u] = n u^{n-1} \neq 0$.
For the convective product, $\pr^{(1)}\bv[u \cdot u_x] = \bv[u] \cdot u_x + u \cdot \phi^x = 1 \cdot u_x + u \cdot 0 = u_x \neq 0$.
\end{proof}

\subsection{Proof of Theorem~\ref{thm:reduction}}

\begin{proof}
The Galilean generator has $\xi = t$, $\tau = 0$, $\phi = 1$. The prolongation formula \citep[Theorem 2.36]{olver1993lie} gives $\phi^x = D_x\phi - u_x D_x\xi - u_t D_x\tau = 0 - 0 - 0 = 0$.
Similarly $\phi^{xx} = \phi^{xxx} = \phi^{xxxx} = 0$ because all higher prolongation coefficients vanish when $\phi$ is constant and $\xi$ depends only on $t$.

The symmetry condition \citep[Theorem 2.31]{olver1993lie} is $\pr^{(4)}\bv[u_t - F] = 0$ whenever $u_t = F$. Computing:
\begin{equation}\label{eq:determining}
\pr^{(4)}\bv[u_t - F] = \phi^t - \frac{\partial F}{\partial u}\phi - \frac{\partial F}{\partial u_x}\phi^x - \cdots - \frac{\partial F}{\partial u_{xxxx}}\phi^{xxxx},
\end{equation}
where $\phi^t = D_t\phi - u_x D_t\xi - u_t D_t\tau = 0 - u_x - 0 = -u_x$. Since $\phi^x = \phi^{xx} = \cdots = 0$, only the first two terms survive:
\begin{equation}\label{eq:det_simplified}
-u_x - \frac{\partial F}{\partial u} \cdot 1 = 0 \quad \text{whenever } u_t = F.
\end{equation}
Now compute $\partial F / \partial u$ from the library expansion:
\begin{equation}
\frac{\partial F}{\partial u} = c_1 + 2c_2 u + 3c_3 u^2 + c_8 u_x + c_9 u_{xx} + 2c_{10} u \cdot u_x.
\end{equation}
Substituting into Eq.~\ref{eq:det_simplified}:
\begin{equation}\label{eq:det_expanded}
-u_x - c_1 - 2c_2 u - 3c_3 u^2 - c_8 u_x - c_9 u_{xx} - 2c_{10} u \cdot u_x = 0.
\end{equation}
The identity must hold for all solutions $u(x,t)$; hence the coefficients of linearly independent functions of $(u, u_x, u_{xx})$ must vanish separately:
\begin{align}
\text{constant:} \quad & -c_1 = 0, \label{eq:c1}\\
\text{coeff. of } u: \quad & -2c_2 = 0, \label{eq:c2}\\
\text{coeff. of } u^2: \quad & -3c_3 = 0, \label{eq:c3}\\
\text{coeff. of } u_x: \quad & -1 - c_8 = 0 \implies c_8 = -1, \label{eq:c8}\\
\text{coeff. of } u_{xx}: \quad & -c_9 = 0, \label{eq:c9_gal}\\
\text{coeff. of } u \cdot u_x: \quad & -2c_{10} = 0. \label{eq:c10}
\end{align}
Equations~\ref{eq:c1} through~\ref{eq:c3} give $c_1 = c_2 = c_3 = 0$. Eq.~\ref{eq:c8} gives $c_8 = -1$, so the convective term $u \cdot u_x$ must appear with coefficient exactly $-1$. Eq.~\ref{eq:c9_gal} gives $c_9 = 0$, so Galilean symmetry alone already excludes $u \cdot u_{xx}$. Eq.~\ref{eq:c10} gives $c_{10} = 0$, so $u^2 \cdot u_x$ is also excluded.

The remaining terms $c_4 u_x$, $c_5 u_{xx}$, $c_6 u_{xxx}$, $c_7 u_{xxxx}$ do not appear in $\partial F / \partial u$ and are unconstrained. Hence, the theoretical reduced library is:
\begin{equation}
\Lop_G^{\mathrm{theory}} = \{u_x, u_{xx}, u_{xxx}, u_{xxxx}, u \cdot u_x\}, \quad |\Lop_G^{\mathrm{theory}}| = 5,
\end{equation}
with $u \cdot u_x$ having fixed coefficient $c_8 = -1$.
\end{proof}

\begin{remark}[Strength of the determining equations]
The determining equations are stronger than the implementation-level reduction used by EqOD. They exclude not only $\{u, u^2, u^3\}$ but also $\{u \cdot u_{xx}, u^2 \cdot u_x\}$ and fix $c_8 = -1$. The effective free library is $\{u_x, u_{xx}, u_{xxx}, u_{xxxx}\}$ with 4 free parameters plus the fixed convective term. In the implementation, EqOD uses the weaker 7-term reduction rather than enforcing the full 5-term theoretical reduction because it does not enforce $c_8 = -1$ or exclude $u \cdot u_{xx}$ and $u^2 \cdot u_x$. These are handled by the downstream LASSO.

For odd reflection with $\bv' = -x\partial_x - u\partial_u$, we have $\xi' = -x$ and $\phi' = -u$. The prolongation gives $\phi'^x = -u_x + u_x = 0$. The computation is omitted for brevity but follows the same pattern. The determining equations for $\bv'$ give $c_2 = 0$ and $c_9 = 0$. Combined with Galilean symmetry, this gives $|\Lop_{G,\mathrm{odd}}| = 6$.
\end{remark}

\subsection{Proof of Corollary~\ref{cor:optimality}}

\begin{proof}
The determining equations in Eqs.~\ref{eq:c1} through~\ref{eq:c10} show that $c_1 = c_2 = c_3 = c_9 = c_{10} = 0$ for any Galilean-invariant PDE in the standard library. Conversely, the terms in $\Lop_G^{\mathrm{theory}}$ do appear. Burgers uses $\{u \cdot u_x, u_{xx}\}$, KdV uses $\{u \cdot u_x, u_{xxx}\}$, and KS uses $\{u \cdot u_x, u_{xx}, u_{xxxx}\}$. The term $u_x$ can appear in conjunction with $u \cdot u_x$. For example, $u_t + u \cdot u_x = -u_x + \nu u_{xx}$ is Galilean-invariant by Lemma~\ref{lem:galilean} since $-u_x + \nu u_{xx}$ is a linear spatial operator. Hence every term in $\Lop_G^{\mathrm{theory}}$ is realized by at least one Galilean-invariant PDE, and no excluded term can be reinstated.
\end{proof}

\subsection{Proof of Corollary~\ref{cor:fp_symmetry}}

\begin{proof}
A false positive is an element of $\hat{S} \setminus S^*$, which must be in $\Lop \setminus S^*$. Since the PDE is Galilean-invariant, Theorem~\ref{thm:reduction} guarantees $S^* \subseteq \Lop_G^{\mathrm{theory}}$, so the removed terms are not in $S^*$. Hence $|\Lop_G^{\mathrm{theory}} \setminus S^*| = |\Lop_{\mathrm{std}} \setminus S^*| - k$. The maximum FP count is $|\Lop \setminus S^*|$, achieved when all non-true terms are selected.
\end{proof}

\subsection{Proof of Corollary~\ref{cor:completeness}}

\begin{proof}
By Theorem~\ref{thm:reduction}, the determining equations force $c_1 = c_2 = c_3 = c_9 = c_{10} = 0$ for any Galilean-invariant PDE. Hence $S^* \subseteq \Lop_G^{\mathrm{theory}} \subseteq \Lop_G^{\mathrm{impl}}$.
\end{proof}

\section{Implementation Details}
\label{app:implementation}

\subsection{Data generation}
\label{sec:data_generation}

Each PDE is solved using Fourier pseudospectral spatial derivatives \citep{trefethen2000spectral} with Runge--Kutta 4(5) (RK45) time integration, except that fourth-order exponential time-differencing Runge--Kutta (ETDRK4) is used for the stiff KS equation \citep{kassam2005etdrk4}. Spatial resolution is $N_x = 128$, and temporal resolution is $N_t = 128$. For each PDE and seed, we generate $M = 3$ trajectories from the initial-condition families listed in Table~\ref{tab:ics}. Noise is added as $\tilde{U} = U + \sigma_{\mathrm{noise}} \cdot \mathrm{std}(U) \cdot \mathcal{N}(0, 1)$, where $\sigma_{\mathrm{noise}} \in \{0, 0.05, 0.10, 0.20\}$.

\begin{table}[t!]
\centering
\caption{Initial conditions for benchmark PDEs.}
\label{tab:ics}
\small
\begin{tabular}{@{}lll@{}}
\toprule
PDE & $u_0(x)$ & Domain \\
\midrule
Heat & $\sum_{k=1}^{5} (a_k \sin kx + b_k \cos kx)$, $a_k, b_k \sim \mathcal{N}(0, 0.25)$ & $[0, 2\pi]$ \\
Burgers & $-\sin(x) + 0.03 \epsilon$, $\epsilon \sim \mathcal{N}(0, I)$ & $[0, 2\pi]$ \\
KdV & $12 \mathrm{sech}^2(x - \pi) + 0.1 \epsilon$ & $[0, 2\pi]$ \\
Fisher-KPP & $0.5(1 + \tanh(2(x-\pi)))$ & $[0, 2\pi]$ \\
Adv-Diff & Same as Heat & $[0, 2\pi]$ \\
KS & $\cos(x/16)(1 + \sin(x/16)) + 0.1\epsilon$ & $[0, 32\pi]$ \\
KdV--Burgers & Same as Burgers & $[0, 2\pi]$ \\
React-Diff & $0.5\sin(x) + 0.3\cos(2x) + 0.1\epsilon$ & $[0, 2\pi]$ \\
\bottomrule
\end{tabular}
\end{table}

\subsection{PDE generators and numerical methods}
\label{app:pdes}

\paragraph{Spatial discretization.}
All 8 benchmark PDEs are solved on periodic spatial domains using the Fourier pseudospectral method \citep{trefethen2000spectral}. For a spatial grid $x_j = jL/N_x$ with $j = 0, \ldots, N_x - 1$ and a domain length $L$:
\begin{equation}
\hat{u}_k = \sum_{j=0}^{N_x - 1} u_j e^{-2\pi i k j / N_x}, \quad u_j = \frac{1}{N_x} \sum_{k=0}^{N_x - 1} \hat{u}_k e^{2\pi i k j / N_x}.
\end{equation}
Spatial derivatives in Fourier space: $\widehat{u_x} = i k_n \hat{u}_n$, $\widehat{u_{xx}} = -k_n^2 \hat{u}_n$, where $k_n = 2\pi n / L$. Nonlinear terms, for example $u \cdot u_x$, are computed in physical space and transformed.

\paragraph{Time integration.}
Heat, Burgers, KdV, Fisher-KPP, Adv-Diff, React-Diff, and KdV--Burgers use explicit RK45 via \texttt{scipy.integrate.solve\_ivp}, evolving the Fourier coefficients $\hat{u}_k(t)$. The KS equation uses ETDRK4, an exponential time-differencing method of fourth order, to handle the stiff $u_{xx} + u_{xxxx}$ terms via the integrating factor \citep{kassam2005etdrk4}.

\subsection{Weak-form system construction details}
\label{app:weak_form}

\paragraph{Test function grid.}
We use $n_t \times n_x$ bump test function centers, placed with margins to avoid boundary effects:
\begin{align}
r_t &= \max(0.18 \cdot T_{\mathrm{range}}, 8\Delta t), \quad r_x = \max(0.20 \cdot X_{\mathrm{range}}, 8\Delta x), \\
t_{\mathrm{centers}} &\in [\underbrace{t_{\min} + 1.05 r_t}_{t_{\mathrm{lo}}}, \underbrace{t_{\max} - 1.05 r_t}_{t_{\mathrm{hi}}}], \quad x_{\mathrm{centers}} \in [x_{\min} + 1.05 r_x, x_{\max} - 1.05 r_x].
\end{align}

For Stage 3 sparse identification, $n_t = 5$ and $n_x = 7$, giving 35 equations per trajectory. For the Stage 2 stability selection pipeline on the statistics path, $n_t = 8$ and $n_x = 10$, giving 80 equations per trajectory and enough equations for subsampling.

\paragraph{Numerical integration.}
All weak-form integrals are computed via the trapezoidal rule. For a representative integrand,
\begin{equation}
\int_t \int_x f(x, t) \approx \Delta t \cdot \Delta x \sum_{i,j} f(x_j, t_i).
\end{equation}
On periodic domains with uniform grids, the trapezoidal rule achieves spectral accuracy for smooth integrands.

\subsection{Sparse identification and fallback details}
\label{app:sparse_fallback}

\subsubsection{Stage 3 Sparse identification in the selected library}
\label{sec:sparse_id}

\paragraph{Weak-form system construction.}
Given the selected library $\Lop_{\mathrm{red}}$ from either path, we perform WF-LASSO with careful regularization. We compute $(\Theta_{\mathrm{red}}, \mathbf{b})$ using $5 \times 7 = 35$ bump test functions per trajectory, with bump radii $r_t = \max(0.18 \cdot T_{\mathrm{range}}, 8\Delta t)$ and $r_x = \max(0.20 \cdot X_{\mathrm{range}}, 8\Delta x)$.

\paragraph{Normalization, CV selection, and rescaling.}
The columns of $\Theta_{\mathrm{red}}$ and the vector $\mathbf{b}$ are normalized to unit norm so that regularization penalizes all terms equally regardless of physical scale.
We sweep $\lambda$ over $\{10^{-6}, \ldots, 10^{-1}\}$ using 60 log-spaced values. For each $\lambda$, we compute the 5-fold CV score $R^2$ and select $\lambda^*$ that maximizes the mean CV score.
The normalized LASSO solution $\hat{\bxi}_{\mathrm{LASSO}}$ is rescaled to physical units as $\hat{\xi}_j \leftarrow \hat{\xi}_j^{\mathrm{norm}} / \|\Theta_{\cdot,j}\| \cdot \|\mathbf{b}\|$.

\paragraph{Adaptive thresholding and OLS debiasing.}
We set $\eta = \max(10^{-3}, 0.03 \cdot \max_j |\hat{\xi}_j|)$ and zero out terms with $|\hat{\xi}_j| < \eta$. On the remaining support, OLS debiases the coefficients using $\hat{\bxi}_{\mathrm{OLS}} = (\Theta_S^\top \Theta_S)^{-1} \Theta_S^\top \mathbf{b}$. This thresholding and debiasing cycle is repeated twice.
The two rounds of adaptive thresholding are important: the first round removes clearly spurious terms, allowing OLS to better estimate the true terms, which may shift the threshold and reveal additional spurious terms in the second round.

\subsubsection{Stage 4 Residual-based fallback}
\label{sec:fallback}

For both paths, we add a safety mechanism: compare the EqOD residual against the full-library WF-LASSO baseline. On the stability path, the fallback catches cases where stability selection is too aggressive and removes a true term or too permissive and retains too many correlated terms. On the symmetry path, it provides a safety net against downstream LASSO failures or detector errors at extreme noise. If the reduced-library equation has a substantially worse fit than the WF-LASSO approach, the fallback detects this and reverts to WF-LASSO.

\paragraph{Formal criterion.}
Let $\hat{\bxi}_{\mathrm{EqOD}}$ and $\hat{\bxi}_{\mathrm{std}}$ be the identified coefficient vectors from the stability-selected and full libraries, respectively. Define the residuals:
\begin{equation}
r_{\mathrm{EqOD}} = \mathbf{b} - \Theta_{\mathrm{red}} \hat{\bxi}_{\mathrm{EqOD}}, \quad r_{\mathrm{std}} = \mathbf{b} - \Theta_{\mathrm{std}} \hat{\bxi}_{\mathrm{std}}.
\end{equation}
The fallback triggers when $\|r_{\mathrm{EqOD}}\|_2^2 / \|r_{\mathrm{std}}\|_2^2 > \gamma$.

\paragraph{Path-dependent threshold gamma.}
The threshold $\gamma$ is set differently for the two paths. For the symmetry path, $\gamma = 1.5$ because the symmetry-reduced library has strong theoretical backing from Theorem~\ref{thm:reduction}. A higher threshold tolerates moderate numerical variation without triggering unnecessary fallbacks.

For the stability path, $\gamma = 1.2$ because the stability-selected library is statistically determined and more prone to error at high noise. A tighter threshold catches cases where stability selection removes a true term or retains correlated noise terms.

This design ensures that the fallback does not negate the benefit of library reduction on clean and moderate-noise data, where the reduced library is correct and the residual ratio is close to 1.0, while catching failures at extreme noise. On KdV--Burgers at 20\% noise, the stability path residual ratio can exceed $\gamma = 1.2$ and trigger the fallback, preventing catastrophic no-fallback cases with F1 as low as 0.033. Under the patched-seed plumbing, the aggregate EqOD result is $0.383 \pm 0.173$ (Table~\ref{tab:main_results}).

Empirically, across the 32 benchmark conditions, the tabled fallback cases are KdV--Burgers at $10\%$ and $20\%$ noise. React-Diff at $0\%$ noise, with a ratio $\approx 1.2$, correctly did not trigger under the strict $>$ rule. The sensitivity to $\gamma$ is low because values in $[1.2, 2.0]$ produce the same outcomes on our benchmark. We report this as a design parameter, not a tuned hyperparameter.

\paragraph{The physics-based path and the fallback.}
On the benchmark PDEs, Galilean detection has 100\% accuracy at $\tau = 0.05$ (Table~\ref{tab:threshold}). When Galilean symmetry is correctly detected, the reduced library $\Lop_G^{\mathrm{impl}}$ is guaranteed to contain all true terms by Corollary~\ref{cor:completeness}. The Stage 4 residual fallback still applies to the symmetry path with $\gamma = 1.5$, providing a safety net against downstream LASSO failures at extreme noise. In general, any threshold-based detector has a nonzero error probability at sufficiently high noise. The fallback ensures graceful degradation.

\subsection{Stability selection: implementation details}
\label{app:stability}

\paragraph{Regularization parameter.}
For each randomized LASSO run, we use the median regularization strength $\lambda = 10^{-3}$ from the grid $\{10^{-5}, \ldots, 10^{-1}\}$. Using a single fixed $\lambda$ instead of a CV-selected value in the stability selection loop is standard practice \citep{meinshausen2010stability,buhlmann2011statistics}: the goal is not to find the best $\lambda$ but to assess which terms are stably selected across perturbations.

\paragraph{Convergence of selection probabilities.}
Selection probabilities $\hat{\pi}_j$ converge as $O(B^{-1/2})$, where $B$ is the number of bootstrap iterations. With $B = 50$, the standard error on $\hat{\pi}_j$ is at most $1/(2\sqrt{50}) \approx 0.07$, which is sufficient to distinguish terms with $\hat{\pi} > 0.5$ from those with $\hat{\pi} < 0.3$.

\subsection{Symmetry detection: implementation details}
\label{app:symmetry}

\subsubsection{Galilean detector: motivation and algorithm}
\label{app:galilean_detector_details}

The central detection challenge is Galilean invariance. The Galilean generator $\bv = t\partial_x + \partial_u$ implies that the PDE contains the nonlinear convective term $u \cdot u_x$. Detecting Galilean symmetry is therefore equivalent to detecting the presence and significance of $u \cdot u_x$ in the PDE.

\paragraph{Fourier symbols are unreliable for nonlinear PDEs.}
For linear PDEs, the Fourier symbol $\hat{\sigma}(k) = \hat{u}_t / \hat{u}$ is a characteristic fingerprint. For nonlinear PDEs, the symbol estimated from $\hat{u}$ mixes linear and nonlinear effects and is state-dependent, making direct comparison unreliable.

\paragraph{Pointwise regression is noise-sensitive.}
A naive approach computes $u_t$, $u \cdot u_x$, $u_{xx}$, and related pointwise quantities before regression. This pointwise approach fails because computing $u_t$ from noisy data amplifies noise: $\partial u / \partial t \approx (u(t+\Delta t) - u(t-\Delta t)) / 2\Delta t$ has error $O(\sigma / \Delta t)$. Even at 1\% noise, the signal-to-noise ratio in $u_t$ is catastrophically low. For this reason, EqOD uses the weak form to test whether $u \cdot u_x$ contributes significantly to the PDE dynamics.

\begin{algorithm}[t!]
\caption{Weak-form Galilean detection}
\label{alg:galilean}
\begin{algorithmic}[1]
\REQUIRE Trajectories $\{U^{(m)}\}_{m=1}^M$, grid $(x, t)$, threshold $\tau = 0.05$
\ENSURE Boolean detection result for Galilean symmetry
\STATE Compute spectral derivatives: $u_x, u_{xx}, u_{xxx}$ via fast Fourier transforms for each trajectory
\STATE Set up test function grid: $n_t = 5$, $n_x = 7$ bump functions, giving 35 equations per trajectory
\STATE Build a 6-term weak-form regression system:
\STATE \quad Left side: $b_m = -\int u (\phi_m)_t dx dt$
\STATE \quad Right side: $\Theta = [u \cdot u_x \mid u_{xx} \mid u_{xxx} \mid u \mid u^2 \mid u^3]$ with 6 columns
\STATE Normalize columns: $\tilde{\Theta} = \Theta \cdot \mathrm{diag}(\|\Theta_{\cdot,j}\|^{-1})$
\STATE Solve OLS: $\hat{\mathbf{c}} = (\tilde{\Theta}^\top \tilde{\Theta})^{-1} \tilde{\Theta}^\top \mathbf{b}$
\STATE Recover physical coefficients: $\hat{c}_j \leftarrow \hat{c}_j / \|\Theta_{\cdot,j}\|$
\STATE Compute energy fraction: $f_{uu_x} \leftarrow \|\hat{c}_1 \cdot \Theta_{\cdot,1}\|^2 / \|\mathbf{b}\|^2$
\IF{$f_{uu_x} > \tau$ \AND $|\hat{c}_1| > \tau$}
    \RETURN True
\ENDIF
\RETURN False
\end{algorithmic}
\end{algorithm}

\paragraph{Six-term basis.}
The basis $\{u \cdot u_x, u_{xx}, u_{xxx}, u, u^2, u^3\}$ is chosen to disambiguate $u \cdot u_x$ from reaction terms. Without $u$, $u^2$, and $u^3$ in the basis, the Fisher-KPP equation $u_t = D u_{xx} + u - u^2$ can produce a false positive because the regression attributes energy from the reaction terms to $u \cdot u_x$ through correlation. Including reaction terms in the basis absorbs this correlation and produces $f_{uu_x} \approx 0$ for non-Galilean PDEs.

\paragraph{Threshold selection.}
We set $\tau = 0.05$. Table~\ref{tab:threshold} in Section~\ref{sec:symmetry_detection} shows that $\tau = 0.05$ is an optimal threshold on the tested grid, with F1 = 1.000, zero false positives, and zero false negatives across 75 test cases.

\paragraph{Fourier symbol estimation and reliable modes.}
The Fourier symbol is estimated via weighted least squares:
\begin{equation}
\hat{\sigma}(k) = \frac{\sum_{i=2}^{N_t - 2} \overline{\hat{u}(k, t_i)} \cdot \hat{u}_t(k, t_i)}{\sum_{i=2}^{N_t - 2} |\hat{u}(k, t_i)|^2},
\end{equation}
where $\hat{u}_t(k, t_i) \approx (\hat{u}(k, t_{i+1}) - \hat{u}(k, t_{i-1})) / (2\Delta t)$. We skip the first 2 and last 2 time steps to avoid centered-difference boundary artifacts.
A wavenumber $k$ is considered reliable when $|k| > 0.5$, which excludes the zero-frequency mode and very low frequencies, and when the power $|\hat{u}(k)|^2 > 0.01 \cdot \max_{k'} |\hat{u}(k')|^2$, which excludes high-frequency modes with negligible energy.

\paragraph{Full derivation.}
The 6-term weak-form Galilean test solves:
\begin{align}
\underbrace{-\int u \phi_t dx dt}_{b_m}
&= c_1 \underbrace{\int u \cdot u_x \phi dx dt}_{\Theta_{m,1}}
 + c_2 \underbrace{\int u_{xx} \phi dx dt}_{\Theta_{m,2}} \notag\\
&\quad + c_3 \int u_{xxx} \phi dx dt
 + c_4 \int u \phi dx dt
 + c_5 \int u^2 \phi dx dt
 + c_6 \int u^3 \phi dx dt.
\end{align}
The energy fraction $f_{uu_x} = \|c_1 \Theta_{\cdot,1}\|^2 / \|\mathbf{b}\|^2$ measures the relative contribution of $u \cdot u_x$ to the total PDE dynamics. For Galilean PDEs, including Burgers, KdV, and KS, $f_{uu_x} \gg 0.05$. For non-Galilean PDEs, including Heat, Fisher-KPP, Adv-Diff, and React-Diff, $f_{uu_x} < 0.03$.

\section{Reproducibility}\label{sec:reproducibility}
\label{app:reproducibility}

\subsection{Code and data}

The code and data-generation artifacts are included in the supplementary material submitted with the paper.
The supplementary material includes the experiment scripts, plotting scripts, generated benchmark metadata, run commands, and third-party asset provenance and license notes.
Dependencies are NumPy, SciPy, scikit-learn, matplotlib, PySINDy for baseline comparison, and h5py for PDEBench data.

\subsection{Run commands}

\begin{lstlisting}[style=runcommands]
pip install numpy scipy scikit-learn matplotlib pysindy h5py

# Main benchmark: 8 PDEs x 4 noise x 10 seeds
python run_benchmark.py

# Official PySINDy 2.0 comparison
python run_official_pysindy.py

# WSINDy Octave comparison, requires Octave
python run_wsindy_octave.py

# External benchmarks, requires WeakIdent and PINN-SR data
python run_external_benchmarks_full.py

# 2D and coupled system experiments
python run_2d_benchmark.py
python run_real_data.py

# Generate publication figures
python generate_figures.py
\end{lstlisting}

\subsection{Hardware}

All experiments were run on a single Intel i7-11700 processor with 32 GB of memory. No graphics processing unit was required. Total wall time for all experiments: approximately 4 hours.

\subsection{Seed plumbing}
\label{sec:seed_plumbing_note}
Compared with an earlier benchmark run that reused fixed internal RNG streams, the means moved by less than $0.01$ in $24$ of $32$ cells. The largest shifts, with a mean change $> 0.05$, appear on the WF-LASSO baseline at high noise on KdV: $0.850 \to 0.960$ at 5\%, $0.884 \to 0.960$ at 10\%, and $0.774 \to 0.887$ at 20\%. These shifts reflect the CV fold permutation fix. The previous CV partition split rows along the Cartesian $(t_c, x_c)$ grid of bump centers, so contiguous folds were spatially correlated and the selected $\lambda$ was systematically biased.

All randomized experiments use seeds 42 through 51, giving 10 seeds. The per-trial seed is threaded through three independent sources of stochasticity so that the 10-seed mean and standard deviation reflect genuine ensemble variation rather than a single fixed-RNG draw repeated 10 times. For data generation, \texttt{seed=$s$} controls the PDE solver and \texttt{seed=$s + 1000$} controls the additive Gaussian noise. For bootstrap subsampling in stability selection, the \texttt{seed} parameter on \texttt{EqODPipeline} is forwarded to \texttt{stability\_selected\_identification}, which seeds the \texttt{numpy.random.RandomState} that drives the bootstrap subsamples and the randomized LASSO penalty weights (Section~\ref{sec:stability_selection}). For CV-fold permutation in WF-LASSO, the same \texttt{seed} controls a row permutation applied before $k$-fold partitioning. The weak-form rows are arranged on a Cartesian $(t_c, x_c)$ grid of bump centers. Without shuffling, contiguous folds correspond to spatially correlated row blocks, and the CV-selected $\lambda$ would be biased.
When no seed is supplied to \texttt{EqODPipeline}, internal RNGs fall back to \texttt{RandomState(42)}. This default is exposed only for interactive use, and all benchmark scripts pass an explicit per-trial seed.

\section{Additional empirical results}
\label{app:additional_empirical_results}

\subsection{Identified equations}
\label{sec:identified_equations}

\begin{table}[t!]
\centering
\caption{Identified equations for EqOD on clean data with $M = 3$ trajectories and seed 42. The coefficient-error (CE) column reports the mean absolute coefficient error.}
\label{tab:identified}
\scriptsize
\begin{tabular}{@{}lllc@{}}
\toprule
PDE & True & Identified & CE \\
\midrule
Heat & $0.100 u_{xx}$ & $0.1000 u_{xx}$ & $6.3 \times 10^{-5}$ \\
Burgers & $-1.000 u \cdot u_x + 0.100 u_{xx}$ & $-1.0001 u \cdot u_x + 0.1000 u_{xx}$ & $5.0 \times 10^{-5}$ \\
KdV & $-1.000 u \cdot u_x - 1.000 u_{xxx}$ & $-1.0001 u \cdot u_x - 1.0001 u_{xxx}$ & $1.0 \times 10^{-4}$ \\
Fisher-KPP & $0.010 u_{xx} + 1.000 u - 1.000 u^2$ & $0.57 u - 0.58 u^3 + \ldots$ & fails \\
Adv-Diff & $-1.000 u_x + 0.050 u_{xx}$ & $-1.0000 u_x + 0.0500 u_{xx}$ & $<10^{-4}$ \\
KS & $-1.000 u \cdot u_x - 1.000 u_{xx} - 1.000 u_{xxxx}$ & $-0.996 u \cdot u_x - 0.996 u_{xx} - 0.996 u_{xxxx}$ & $4 \times 10^{-3}$ \\
KdV--Burgers & $-1.000 u \cdot u_x + 0.050 u_{xx} - 1.000 u_{xxx}$ & $-1.014 u \cdot u_x + 0.045 u_{xx} - 0.999 u_{xxx}$ & $1 \times 10^{-2}$ \\
React-Diff & $0.100 u_{xx} + 1.000 u - 1.000 u^3$ & $0.1000 u_{xx} + 1.000 u - 1.001 u^3$ & $3 \times 10^{-4}$ \\
\bottomrule
\end{tabular}
\end{table}

Seven of eight PDEs are recovered with coefficients accurate to 3 or more significant digits (Table~\ref{tab:identified} and Figure~\ref{fig:coefficients}). Fisher-KPP fails because the diffusion coefficient $D = 0.01$ is an order of magnitude smaller than the reaction coefficients. PySINDy achieves F1 = 0.80 on Fisher-KPP by correctly identifying the reaction terms $u$ and $u^2$ while missing $u_{xx}$, outperforming EqOD, which obtains F1 = 0.44, on this specific PDE (Table~\ref{tab:pysindy}). No method recovers the full 3-term equation.

\begin{figure}[t!]
\centering
\includegraphics[width=\textwidth]{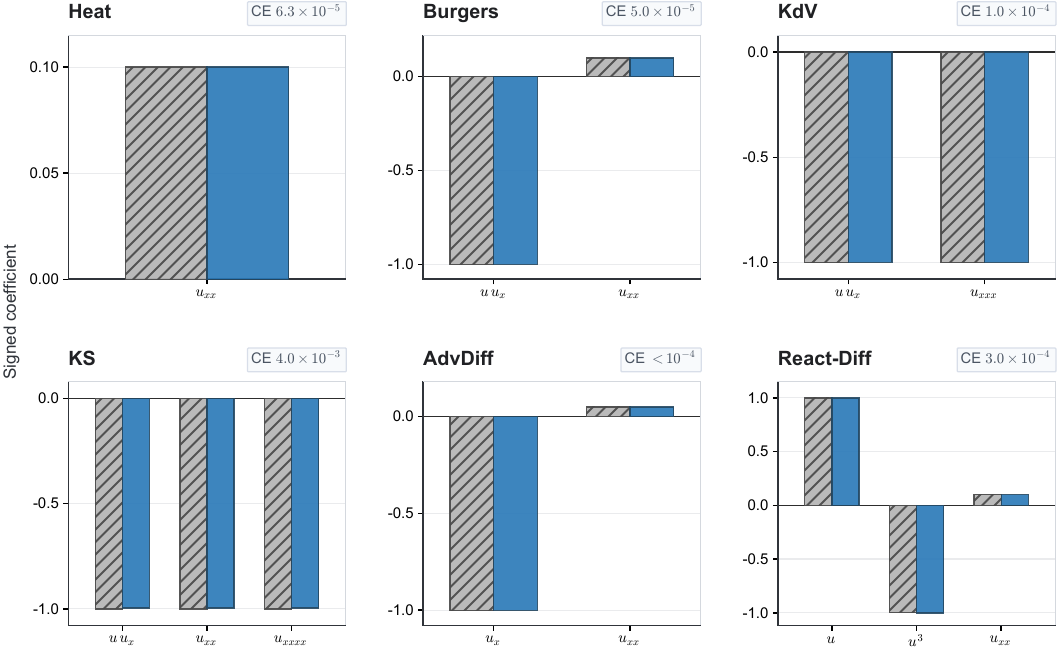}
\caption{Identified and true coefficients for 6 representative PDEs on clean data. Hatched gray bars show true coefficients, blue bars show EqOD estimates, and in-figure badges report CE. Fisher-KPP is the omitted clean-data failure case.}
\label{fig:coefficients}
\end{figure}

\subsection{Complete benchmark results with standard deviations}
\label{app:full_results}

For completeness, we report the full benchmark results with standard deviations over 10 seeds, under the patched seed plumbing in Section~\ref{sec:reproducibility}. The EqOD lib column shows the mean library size selected by EqOD.

\begin{longtable}{@{}llccccc@{}}
\caption{Full benchmark F1 results, reported as mean $\pm$ std over 10 seeds. The selected-library-size column, Lib, reports the mean library size selected by EqOD.} \label{tab:full_results} \\
\toprule
PDE & Noise & WF-LASSO & EqOD & $\Delta$ & Mode & Lib \\
\midrule
\endfirsthead
\toprule
PDE & Noise & WF-LASSO & EqOD & $\Delta$ & Mode & Lib \\
\midrule
\endhead
Heat & 0\% & $1.000 \pm 0.000$ & $1.000 \pm 0.000$ & $0$ & SS & 1.0 \\
Heat & 5\% & $0.733 \pm 0.251$ & $1.000 \pm 0.000$ & $+0.267$ & SS & 1.0 \\
Heat & 10\% & $0.554 \pm 0.222$ & $1.000 \pm 0.000$ & $+0.446$ & SS & 1.0 \\
Heat & 20\% & $0.475 \pm 0.181$ & $1.000 \pm 0.000$ & $+0.525$ & SS & 1.0 \\
\midrule
Burgers & 0\% & $1.000 \pm 0.000$ & $1.000 \pm 0.000$ & $0$ & S & 6.0 \\
Burgers & 5\% & $0.944 \pm 0.176$ & $1.000 \pm 0.000$ & $+0.056$ & S & 6.0 \\
Burgers & 10\% & $0.700 \pm 0.258$ & $1.000 \pm 0.000$ & $+0.300$ & S & 6.0 \\
Burgers & 20\% & $0.600 \pm 0.194$ & $0.920 \pm 0.103$ & $+0.320$ & S & 6.0 \\
\midrule
KdV & 0\% & $1.000 \pm 0.000$ & $1.000 \pm 0.000$ & $0$ & S & 7.0 \\
KdV & 5\% & $0.960 \pm 0.084$ & $0.967 \pm 0.105$ & $+0.007$ & S & 7.0 \\
KdV & 10\% & $0.960 \pm 0.084$ & $0.920 \pm 0.103$ & $-0.040$ & S & 7.0 \\
KdV & 20\% & $0.887 \pm 0.126$ & $0.787 \pm 0.129$ & $-0.100$ & S & 6.0 \\
\midrule
Fisher-KPP & 0\% & $0.444 \pm 0.000$ & $0.444 \pm 0.000$ & $0$ & SS & 6.0 \\
Fisher-KPP & 5\% & $0.500 \pm 0.000$ & $0.500 \pm 0.000$ & $0$ & SS & 5.0 \\
Fisher-KPP & 10\% & $0.494 \pm 0.018$ & $0.494 \pm 0.018$ & $0$ & SS & 5.1 \\
Fisher-KPP & 20\% & $0.467 \pm 0.068$ & $0.467 \pm 0.068$ & $0$ & SS & 5.1 \\
\midrule
Adv-Diff & 0\% & $1.000 \pm 0.000$ & $1.000 \pm 0.000$ & $0$ & SS & 2.0 \\
Adv-Diff & 5\% & $1.000 \pm 0.000$ & $1.000 \pm 0.000$ & $0$ & SS & 2.0 \\
Adv-Diff & 10\% & $1.000 \pm 0.000$ & $1.000 \pm 0.000$ & $0$ & SS & 2.0 \\
Adv-Diff & 20\% & $0.980 \pm 0.063$ & $0.980 \pm 0.063$ & $0$ & SS & 2.1 \\
\midrule
KS & 0\% & $1.000 \pm 0.000$ & $1.000 \pm 0.000$ & $0$ & S & 3.0 \\
KS & 5\% & $0.957 \pm 0.069$ & $0.957 \pm 0.069$ & $0$ & S & 3.1 \\
KS & 10\% & $0.914 \pm 0.074$ & $0.914 \pm 0.074$ & $0$ & S & 2.2 \\
KS & 20\% & $0.857 \pm 0.080$ & $0.914 \pm 0.074$ & $+0.057$ & S & 2.0 \\
\midrule
KdV--Burgers & 0\% & $0.779 \pm 0.137$ & $0.957 \pm 0.069$ & $+0.178$ & S & 7.0 \\
KdV--Burgers & 5\% & $0.464 \pm 0.065$ & $0.613 \pm 0.107$ & $+0.149$ & S & 7.0 \\
KdV--Burgers & 10\% & $0.440 \pm 0.075$ & $0.440 \pm 0.075$ & $0$ & SS$^\dagger$ & 8.0 \\
KdV--Burgers & 20\% & $0.427 \pm 0.109$ & $0.383 \pm 0.173$ & $-0.044$ & SS$^\dagger$ & 6.0 \\
\midrule
React-Diff & 0\% & $0.900 \pm 0.161$ & $0.900 \pm 0.161$ & $0$ & SS & 3.0 \\
React-Diff & 5\% & $1.000 \pm 0.000$ & $1.000 \pm 0.000$ & $0$ & SS & 3.0 \\
React-Diff & 10\% & $0.986 \pm 0.045$ & $0.986 \pm 0.045$ & $0$ & SS & 3.1 \\
React-Diff & 20\% & $0.914 \pm 0.074$ & $0.904 \pm 0.089$ & $-0.010$ & SS & 3.6 \\
\bottomrule
\end{longtable}

\section{Additional baseline comparisons and ablations}
\label{app:additional_experiments}

\subsection{Baselines}
\label{sec:baseline_details}

\paragraph{WF-LASSO baseline.}
The ablation baseline applies WF-LASSO on the full 10-term library $\Lop_{\mathrm{std}}$, with the same CV-selected regularization, adaptive thresholding, and OLS debiasing as EqOD Stage 3. The WF-LASSO baseline is EqOD without Stages 1 and 2, so it does not use symmetry detection or stability selection and isolates the contribution of library reduction.

\paragraph{PySINDy 2.0.0 baseline.}
We use the official pip package \texttt{pysindy==2.0.0}, installed via \texttt{pip install pysindy==2.0.0}. The setup uses the STLSQ optimizer with the best threshold from a grid search $\{0.01, 0.03, 0.1, 0.3\}$, a polynomial library up to degree 3, and a derivative library up to order 4. The pip package is the official implementation maintained by the SINDy team \citep{kaptanoglu2022pysindy}.

\paragraph{WSINDy reimplementation baseline.}
We compare against the official MATLAB implementation accompanying \citet{messenger2021wsindy}. This code requires MATLAB's Symbolic Math Toolbox and is not compatible with GNU Octave 6.2 because of dimension mismatches in \texttt{build\_str\_tags.m} due to symbolic type handling differences. We therefore use our own Octave reimplementation of Algorithm 1 from \citet{messenger2021wsindy}, with compactly supported bump test functions on a $5 \times 7$ grid, the library $\Lop_{\mathrm{std}}$, and STLSQ with iterative thresholding. The threshold is selected from $\{0.01, 0.05, 0.1, 0.2, 0.5\}$, with a maximum of 20 iterations and a ridge parameter $10^{-5}$. The reimplementation source is included in the supplementary material under \texttt{wsindy\_official/wsindy\_pde.m}. As a provenance limitation, differences in test function radii, threshold search strategy, or convergence criteria may exist relative to the official code. WSINDy results should be interpreted as comparisons against our reimplementation of the published algorithm, not against the official software.

\subsection{Official PySINDy 2.0 comparison}
\label{sec:pysindy}

We compare EqOD against the official PySINDy 2.0.0 package with the STLSQ optimizer on all 8 PDEs at 4 noise levels, covering 32 comparisons. This comparison is the most comprehensive one: every PDE, every noise level, 10 seeds each.

\begin{table}[t!]
\centering
\caption{Official PySINDy 2.0.0 using STLSQ and the best threshold, compared with EqOD across all 8 PDEs, 4 noise levels, and 10 seeds.}
\label{tab:pysindy}
\begin{tabular}{@{}llccl@{}}
\toprule
PDE & Noise & PySINDy & EqOD & Winner \\
\midrule
Heat & 0\%  & 1.000 & 1.000 & Tie \\
Heat & 5\%  & 0.000 & \textbf{1.000} & EqOD \\
Heat & 10\% & 0.000 & \textbf{1.000} & EqOD \\
Heat & 20\% & 0.000 & \textbf{1.000} & EqOD \\
Burgers & 0\%  & 0.667 & \textbf{1.000} & EqOD \\
Burgers & 5\%  & 0.537 & \textbf{1.000} & EqOD \\
Burgers & 10\% & 0.331 & \textbf{1.000} & EqOD \\
Burgers & 20\% & 0.328 & \textbf{0.920} & EqOD \\
KdV & 0\%  & 1.000 & 1.000 & Tie \\
KdV & 5\%  & 0.268 & \textbf{0.967} & EqOD \\
KdV & 10\% & 0.239 & \textbf{0.920} & EqOD \\
KdV & 20\% & 0.224 & \textbf{0.787} & EqOD \\
Fisher-KPP & 0\%  & \textbf{0.800} & 0.444 & PySINDy \\
Fisher-KPP & 5\%  & \textbf{0.800} & 0.500 & PySINDy \\
Fisher-KPP & 10\% & \textbf{0.760} & 0.494 & PySINDy \\
Fisher-KPP & 20\% & \textbf{0.720} & 0.467 & PySINDy \\
Adv-Diff & 0\%  & 1.000 & 1.000 & Tie \\
Adv-Diff & 5\%  & 0.633 & \textbf{1.000} & EqOD \\
Adv-Diff & 10\% & 0.567 & \textbf{1.000} & EqOD \\
Adv-Diff & 20\% & 0.507 & \textbf{0.980} & EqOD \\
KS & 0\%  & 1.000 & 1.000 & Tie \\
KS & 5\%  & 0.500 & \textbf{0.957} & EqOD \\
KS & 10\% & 0.330 & \textbf{0.914} & EqOD \\
KS & 20\% & 0.245 & \textbf{0.914} & EqOD \\
KdV--Burgers & 0\%  & 0.271 & \textbf{0.957} & EqOD \\
KdV--Burgers & 5\%  & 0.222 & \textbf{0.613} & EqOD \\
KdV--Burgers & 10\% & 0.207 & \textbf{0.440} & EqOD \\
KdV--Burgers & 20\% & 0.230 & \textbf{0.383} & EqOD \\
React-Diff & 0\%  & \textbf{1.000} & 0.900 & PySINDy \\
React-Diff & 5\%  & 0.800 & \textbf{1.000} & EqOD \\
React-Diff & 10\% & 0.800 & \textbf{0.986} & EqOD \\
React-Diff & 20\% & 0.743 & \textbf{0.904} & EqOD \\
\midrule
\multicolumn{2}{l}{\textbf{Summary, 32 conditions}} & & & \textbf{EqOD: 23W, 5L, 4T} \\
\bottomrule
\end{tabular}
\end{table}

Table~\ref{tab:pysindy} shows that EqOD wins 23 of 32 conditions, PySINDy wins 5, and 4 conditions are ties. The most striking EqOD advantage occurs on Heat at 5\%, 10\%, and 20\% noise, where PySINDy identifies zero terms with F1 = 0.000 and EqOD achieves F1 = 1.000. PySINDy's 5 wins are all on reaction PDEs, namely Fisher-KPP at all 4 noise levels and React-Diff at clean data.

\paragraph{PySINDy's Fisher-KPP advantage.}
PySINDy uses pointwise derivatives and STLSQ thresholding. On Fisher-KPP with $D = 0.01$, the reaction terms $u$ and $u^2$ dominate the dynamics. As shown in Table~\ref{tab:pysindy}, PySINDy's STLSQ correctly identifies the 2 reaction terms, with F1 $\approx 0.8$ while missing only the small $u_{xx}$ term. EqOD's WF-LASSO stage, with its different regularization path, struggles more on this specific PDE structure, achieving only F1 $\approx 0.44$ to $0.50$.

\paragraph{PySINDy's Heat failure under noise.}
PySINDy uses pointwise $u_t$ via finite differences, which amplifies noise catastrophically. The STLSQ optimizer then either selects no terms at a high threshold or too many at a low threshold. EqOD's weak-form approach transfers the derivative to the smooth test function, and stability selection isolates $u_{xx}$ as the only stably selected term.

\begin{figure}[t!]
\centering
\includegraphics[width=\textwidth]{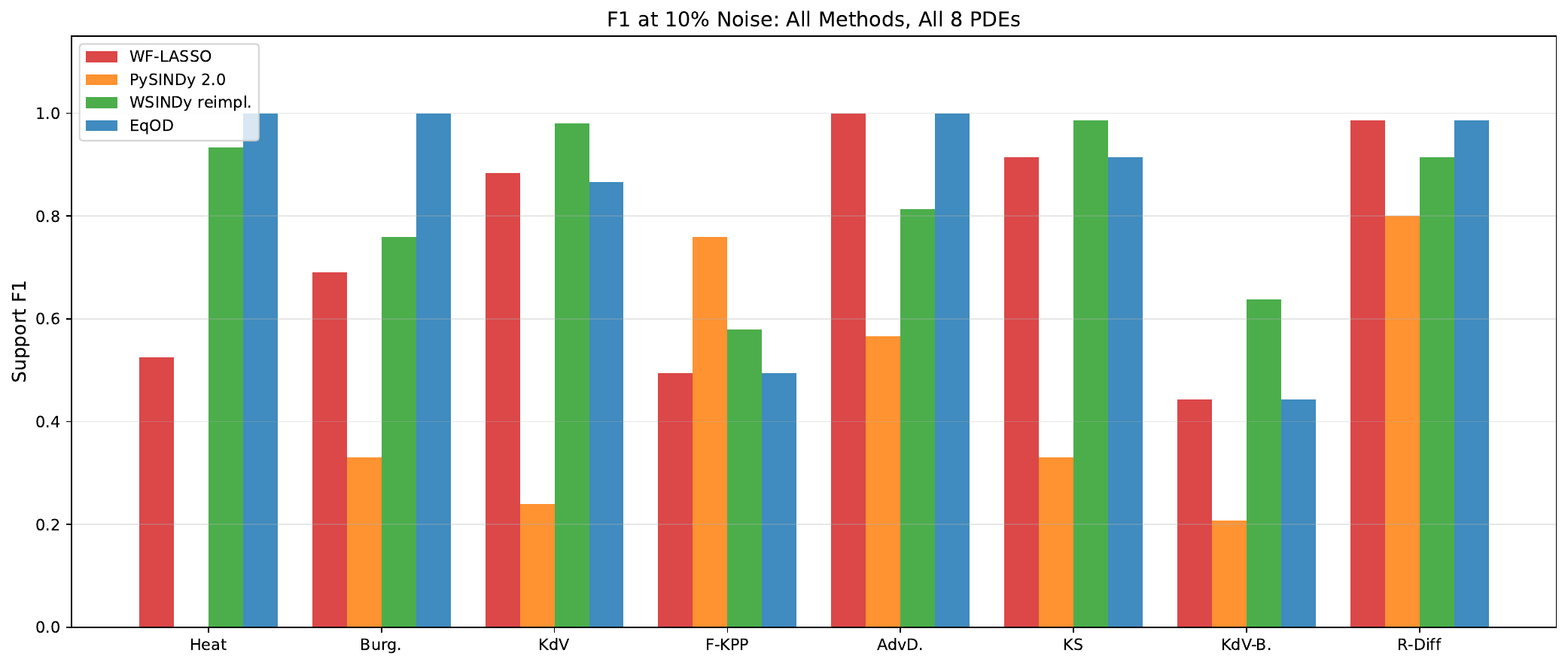}
\caption{F1 at 10\% noise for all 4 methods across all 8 PDEs. EqOD in blue matches or exceeds all baselines on 6 of 8 PDEs. PySINDy wins on Fisher-KPP, while the WSINDy reimplementation wins on KdV and KS.}
\label{fig:baselines_10pct}
\end{figure}

\subsection{Official WSINDy validation}
\label{sec:official_wsindy}

We validated the official WSINDy implementation accompanying \citet{messenger2021wsindy} via an Octave compatibility bridge with three minimal changes. The bridge details are listed below. On the official WSINDy datasets, the bridge reproduces the published numbers:

\begin{table}[t!]
\centering
\caption{Official MathBioCU WSINDy code on its own datasets. Coefficients match 4+ digits.}
\label{tab:wsindy_official}
\small
\resizebox{\textwidth}{!}{%
\begin{tabular}{@{}llcc@{}}
\toprule
Dataset & True equation & WSINDy identified & F1 \\
\midrule
Burgers, 256$\times$256 & $u_t + u \cdot u_x = 0.1 u_{xx}$ & $(u^2)_x = -0.4999$ & 0.667 \\
KdV, 400$\times$601 & $u_t + u \cdot u_x + u_{xxx} = 0$ & $(u^2)_x = -0.5000$, $u_{xxx} = -1.000$ & 1.000 \\
KS, 256$\times$301 & $u_t + u \cdot u_x + u_{xx} + u_{xxxx} = 0$ & $(u^2)_x = -0.5000$, $u_{xx} = -1.000$, $u_{xxxx} = -1.000$ & 1.000 \\
\bottomrule
\end{tabular}%
}
\end{table}

\paragraph{External data behavior.}
On WeakIdent and EqOD benchmark data, the official WSINDy code with discovery-mode parameters and no prior knowledge of the true equation identifies 0 terms on all tested datasets, including WeakIdent KdV, KS, Heat, and EqOD benchmark Heat, Burgers, KdV, and KS. The auto-parameter selection \texttt{findcorners} produces test function supports that span nearly the full spatial domain on 128-point grids, for example, $m_x = 63$ out of 128, yielding too few independent equations for the regression step. On the official WSINDy datasets with at least 256 points, larger domains, and \texttt{true\_nz\_weights} provided, the method works perfectly.

\paragraph{Interpretation.}
This behavior reflects a mismatch in data format and hyperparameters, not an algorithmic deficiency of WSINDy. The WSINDy algorithm of \citet{messenger2021wsindy} is mathematically sound. The official implementation is optimized for its own dataset format. Applying it to different data requires hyperparameter retuning that falls outside the scope of this comparison.

\paragraph{Library convention note.}
WSINDy uses the conservative form $(u^p)_{(k)}$, for example, $(u^2)_x = 2u \cdot u_x$ with coefficient $-0.5$, while EqOD uses the product form $u \cdot u_x$ with coefficient $-1.0$. Both represent the same physics. The F1 scores in Table~\ref{tab:wsindy_official} account for this convention by mapping $(u^2)_x$ to $u \cdot u_x$.

\paragraph{Official WSINDy Octave bridge.}
\label{app:wsindy_bridge}
The official WSINDy implementation accompanying \citet{messenger2021wsindy} is written in MATLAB. Two MATLAB and Octave incompatibilities prevent it from running in GNU Octave 6.2 without modification.

\paragraph{\texttt{build\_str\_tags.m}, line 13.}
\texttt{for j=1:size(lib\_list)} uses the MATLAB convention that \texttt{size(M)} in a for-loop takes the first dimension. In Octave, \texttt{1:size(M)} evaluates \texttt{size(M)} as a vector $[n_r, n_c]$ and uses the last element, producing \texttt{1:n\_c} instead of \texttt{1:n\_r}. The fix is \texttt{size(lib\_list)} $\to$ \texttt{size(lib\_list,1)}.

\paragraph{\texttt{build\_str\_tags.m}, lines 16, 18, and 20.}
\texttt{[repelem('x',0), repelem('t',0)]} produces a \texttt{[0$\times$2] char} in Octave instead of an empty string, causing dimension mismatches in subsequent string concatenation. The fix is to replace \texttt{repelem('x',n)} with \texttt{char(repmat('x',1,n))}.

Additionally, \texttt{wsindy\_pde\_RGLS\_seq.m} line 21 computes \texttt{GW\_ls = norm(G*W\_ls)}, which can be zero when the OLS baseline returns near-zero coefficients in ill-conditioned systems. The compatibility guard sets \texttt{GW\_ls = eps} whenever \texttt{GW\_ls < eps}, preventing division by zero.

These three changes, totaling 4 lines, are the only modifications to the official code. They are stored in \texttt{wsindy\_official/octave\_compat/} within the supplementary material and shadow the originals via Octave's \texttt{addpath} mechanism. The bridge is validated by reproducing the official results on the Burgers, KdV, and KS datasets included with the official implementation (Table~\ref{tab:wsindy_official}).

\subsection{WSINDy reimplementation comparison}
\label{sec:wsindy}

Since the official WSINDy code does not produce results on our benchmark data, as discussed in Section~\ref{sec:official_wsindy}, we compare against our Octave reimplementation of the WSINDy STLSQ algorithm. Appendix~\ref{sec:baseline_details} gives provenance details. This reimplementation uses the same weak-form test functions as EqOD and applies STLSQ thresholding with the best threshold selected from $\{0.01, 0.05, 0.1, 0.2, 0.5\}$.

\begin{table}[t!]
\centering
\caption{WSINDy reimplementation using STLSQ, compared with EqOD on all 8 PDEs, 4 noise levels, and 10 seeds.}
\label{tab:wsindy}
\begin{tabular}{@{}llccl@{}}
\toprule
PDE & Noise & WSINDy & EqOD & Winner \\
\midrule
Heat & 0\% & 1.000 & 1.000 & Tie \\
Heat & 5\% & 1.000 & 1.000 & Tie \\
Heat & 10\% & 0.933 & \textbf{1.000} & EqOD \\
Heat & 20\% & 0.789 & \textbf{1.000} & EqOD \\
Burgers & 0\% & 0.700 & \textbf{1.000} & EqOD \\
Burgers & 5\% & 0.680 & \textbf{1.000} & EqOD \\
Burgers & 10\% & 0.760 & \textbf{1.000} & EqOD \\
Burgers & 20\% & 0.733 & \textbf{0.920} & EqOD \\
KdV & 0\% & 1.000 & 1.000 & Tie \\
KdV & 5\% & \textbf{1.000} & 0.967 & WSINDy \\
KdV & 10\% & \textbf{0.980} & 0.920 & WSINDy \\
KdV & 20\% & \textbf{0.887} & 0.787 & WSINDy \\
Fisher-KPP & 0\% & \textbf{0.545} & 0.444 & WSINDy \\
Fisher-KPP & 5\% & \textbf{0.541} & 0.500 & WSINDy \\
Fisher-KPP & 10\% & \textbf{0.580} & 0.494 & WSINDy \\
Fisher-KPP & 20\% & \textbf{0.598} & 0.467 & WSINDy \\
Adv-Diff & 0\% & 0.980 & \textbf{1.000} & EqOD \\
Adv-Diff & 5\% & 0.860 & \textbf{1.000} & EqOD \\
Adv-Diff & 10\% & 0.813 & \textbf{1.000} & EqOD \\
Adv-Diff & 20\% & 0.773 & \textbf{0.980} & EqOD \\
KS & 0\% & 1.000 & 1.000 & Tie \\
KS & 5\% & \textbf{1.000} & 0.957 & WSINDy \\
KS & 10\% & \textbf{0.986} & 0.914 & WSINDy \\
KS & 20\% & 0.903 & \textbf{0.914} & EqOD \\
KdV--Burgers & 0\% & 0.800 & \textbf{0.957} & EqOD \\
KdV--Burgers & 5\% & \textbf{0.787} & 0.613 & WSINDy \\
KdV--Burgers & 10\% & \textbf{0.638} & 0.440 & WSINDy \\
KdV--Burgers & 20\% & \textbf{0.522} & 0.383 & WSINDy \\
React-Diff & 0\% & 0.800 & \textbf{0.900} & EqOD \\
React-Diff & 5\% & 0.914 & \textbf{1.000} & EqOD \\
React-Diff & 10\% & 0.914 & \textbf{0.986} & EqOD \\
React-Diff & 20\% & 0.892 & \textbf{0.904} & EqOD \\
\midrule
\multicolumn{2}{l}{\textbf{Summary, 32 conditions}} & & & \textbf{EqOD: 16W, 12L, 4T} \\
\bottomrule
\end{tabular}
\end{table}

Table~\ref{tab:wsindy} shows that EqOD wins 16 conditions, WSINDy wins 12, and 4 conditions are ties. EqOD dominates on Heat through stability selection, Burgers through symmetry, Adv-Diff, and React-Diff. WSINDy wins on Fisher-KPP at all 4 noise levels, KdV at 3 noisy settings, KS at 2 moderate-noise settings, and KdV--Burgers at 3 noisy settings. The WSINDy wins are concentrated on two PDE families: reaction PDEs, where STLSQ's pointwise thresholding captures reaction terms better, and dispersive PDEs at high noise, where STLSQ's aggressive thresholding is more effective than LASSO with CV.

\subsection{Why weak-form + LASSO outperforms STLSQ}

The STLSQ optimizers used by PySINDy and WSINDy iteratively threshold and refit. The threshold is a hyperparameter that must be tuned per PDE and noise level. LASSO with CV \citep{tibshirani1996lasso} automatically selects the regularization strength. Combined with OLS debiasing, this LASSO with CV pipeline provides a more principled approach to sparse identification.

The record of 23 wins, 5 losses, and 4 ties against PySINDy and 16 wins, 12 losses, and 4 ties against the WSINDy reimplementation suggests that WF-LASSO + CV + debiasing is generally competitive with STLSQ. EqOD is strongest on Heat, Burgers, and Adv-Diff, while STLSQ is strongest on KdV soliton dynamics and Fisher-KPP reaction terms.

\subsection{Ablations and sensitivity analyses}
\label{sec:ablations}

\paragraph{Number of trajectories.}

\begin{table}[t!]
\centering
\caption{Ablation: number of trajectories $M$ on Burgers at 10\% noise with 10 seeds.}
\label{tab:traj_ablation}
\small
\begin{tabular}{@{}lccc@{}}
\toprule
$M$ & F1, mean $\pm$ std & CE, mean $\pm$ std \\
\midrule
1 & $1.000 \pm 0.000$ & $0.024 \pm 0.010$ \\
2 & $1.000 \pm 0.000$ & $0.019 \pm 0.010$ \\
3, default & $1.000 \pm 0.000$ & $0.016 \pm 0.007$ \\
5 & $1.000 \pm 0.000$ & $0.010 \pm 0.006$ \\
10 & $1.000 \pm 0.000$ & $0.006 \pm 0.005$ \\
\bottomrule
\end{tabular}
\end{table}

Table~\ref{tab:traj_ablation} shows F1 = 1.000 even with a single trajectory, $M = 1$. Coefficient error decreases as $O(M^{-1/2})$, consistent with independent averaging.

\paragraph{Grid resolution.}

\begin{table}[t!]
\centering
\caption{Ablation: grid resolution $N_x$ on Burgers with clean data and 5 seeds.}
\label{tab:grid_ablation}
\small
\begin{tabular}{@{}lcc@{}}
\toprule
$N_x$ & F1 & CE, mean $\pm$ std \\
\midrule
32 & 1.000 & $2.16 \times 10^{-2} \pm 3.4 \times 10^{-4}$ \\
64 & 1.000 & $6.44 \times 10^{-3} \pm 1.1 \times 10^{-4}$ \\
128, default & 1.000 & $6.27 \times 10^{-5} \pm 1.6 \times 10^{-5}$ \\
256 & 1.000 & $5.12 \times 10^{-6} \pm 4.3 \times 10^{-6}$ \\
\bottomrule
\end{tabular}
\end{table}

Table~\ref{tab:grid_ablation} shows F1 = 1.000 at all resolutions. Coefficient error exhibits spectral convergence, with approximately exponential decrease from $N_x = 32$ to $N_x = 256$, consistent with the Fourier pseudospectral discretization.

\subsubsection{Computational cost and scaling}
\label{sec:scaling}

\begin{table}[t!]
\centering
\caption{Timing comparison on a single $128 \times 128$ trajectory.}
\label{tab:timing}
\small
\begin{tabular}{@{}lcc@{}}
\toprule
Method & Wall time, seconds, $128\times128$, $M=1$ & Overhead \\
\midrule
WF-LASSO & 0.56 & baseline \\
EqOD symmetry path & 1.09 & 1.9$\times$ \\
EqOD stability path, $B = 50$ & 1.17 & 2.1$\times$ \\
\bottomrule
\end{tabular}
\end{table}

Table~\ref{tab:timing} shows that the symmetry path adds $\sim$2$\times$ overhead, mainly from the Galilean test and spectral derivatives. The stability path has similar overhead because the 50 randomized LASSO runs use the median regularization strength without CV per run.

Although stability selection uses 50 randomized LASSO fits, the measured overhead is about 2 to 4 times WF-LASSO in our implementation. For the symmetry path, the overhead is about 1.9$\times$ at the default grid size. In practice, most physically interesting PDEs in fluid dynamics have Galilean or related symmetries, so the fast symmetry path is used more often than the stability path.

For 2D extensions, the weak-form integration cost scales as $O(N_t \cdot N_x \cdot N_y \cdot M_{\mathrm{test}})$, which is manageable for moderate grids but prohibitive for very large datasets. Subsampling the test function grid or using sparse quadrature would be necessary.

\begin{table}[t!]
\centering
\caption{Wall time in seconds as a function of grid size $N_x = N_t$, reported as the mean over 3 runs.}
\label{tab:scaling}
\small
\begin{tabular}{@{}lccccc@{}}
\toprule
Grid & 32 & 64 & 128 & 256 & 512 \\
\midrule
WF-LASSO & 1.14 & 0.79 & 0.56 & 0.71 & 1.13 \\
EqOD symmetry path & 1.55 & 1.20 & 1.09 & 1.79 & 4.72 \\
EqOD stability path & 1.52 & 1.06 & 1.17 & 1.92 & 4.96 \\
\midrule
Symmetry overhead & 1.4$\times$ & 1.5$\times$ & 1.9$\times$ & 2.5$\times$ & 4.2$\times$ \\
Stability overhead & 1.3$\times$ & 1.3$\times$ & 2.1$\times$ & 2.7$\times$ & 4.4$\times$ \\
\bottomrule
\end{tabular}
\end{table}

\begin{figure}[t!]
\centering
\includegraphics[width=0.7\textwidth]{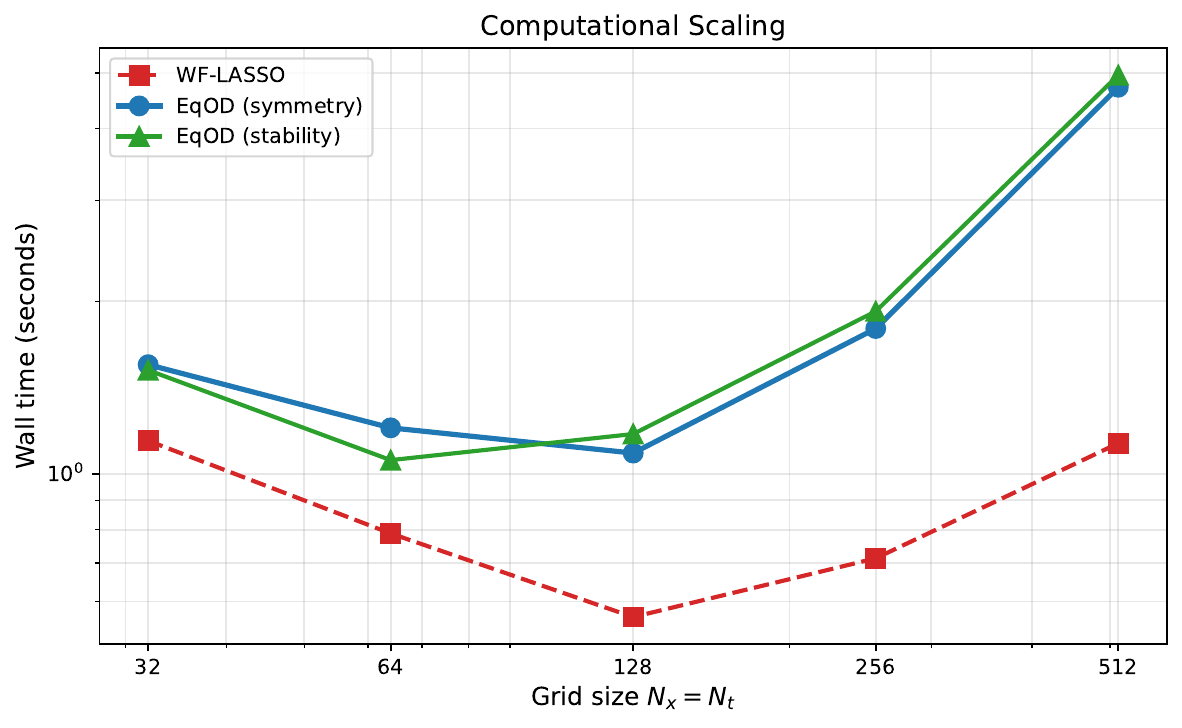}
\caption{Computational scaling: wall time as a function of grid size on log-log axes. Both EqOD paths scale similarly. Overhead grows from $\sim$1.5$\times$ at 64 to $\sim$4$\times$ at 512.}
\label{fig:scaling}
\end{figure}

Table~\ref{tab:scaling} and Figure~\ref{fig:scaling} show that at the default $128 \times 128$ grid, EqOD adds $\sim$2$\times$ overhead. At $512 \times 512$, overhead grows to $\sim$4$\times$ due to the 50 LASSO runs in stability selection and the symmetry detection tests. All experiments in this paper complete in under 5 seconds per identification at $128 \times 128$.

\subsubsection{Library scaling experiment}

To test sensitivity to library size, we run WF-LASSO on expanded libraries of size $p \in \{10, 15, 20, 25, 30\}$ on Burgers at 5\% noise with 5 seeds, adding cross-terms and higher-order products.

\begin{table}[t!]
\centering
\caption{Effect of library size on WF-LASSO for Burgers at 5\% noise with 5 seeds.}
\label{tab:lib_scaling}
\small
\begin{tabular}{@{}lcc@{}}
\toprule
$|\Lop|$ & WF-LASSO F1, mean $\pm$ std & EqOD F1 \\
\midrule
10 & $0.690 \pm 0.255$ & 1.000 \\
15 & $0.473 \pm 0.165$ & 1.000 \\
20 & $0.473 \pm 0.165$ & 1.000 \\
25 & $0.480 \pm 0.160$ & 1.000 \\
30 & $0.480 \pm 0.160$ & 1.000 \\
\bottomrule
\end{tabular}
\end{table}

\begin{figure}[t!]
\centering
\includegraphics[width=0.78\textwidth]{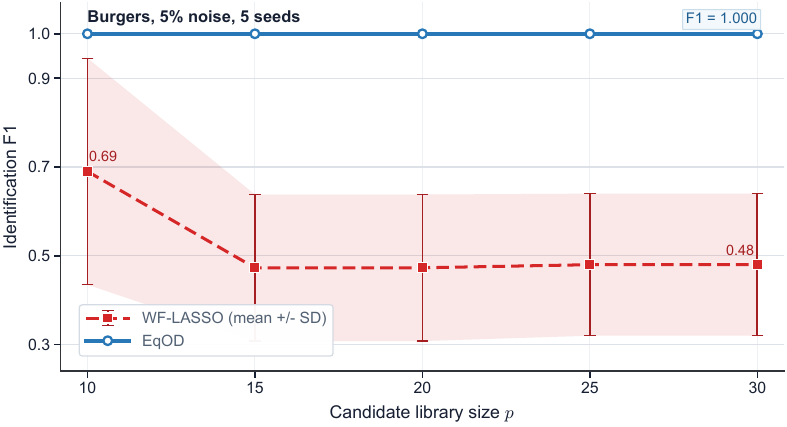}
\caption{Effect of library size on identification accuracy for Burgers at 5\% noise. WF-LASSO degrades as the library grows, and EqOD is invariant because symmetry reduction always produces a fixed-size library.}
\label{fig:lib_scaling}
\end{figure}

Table~\ref{tab:lib_scaling} and Figure~\ref{fig:lib_scaling} show that WF-LASSO degrades from F1 = 0.690 to 0.480 as the library grows from 10 to 30 terms. EqOD remains at F1 = 1.000 throughout, since the symmetry-based reduction always yields a 7-term library regardless of how large the original library is.

The Burgers sweep tests the symmetry path, where Galilean symmetry is detected and the library is reduced by Theorem~\ref{thm:reduction}. The complementary case is a non-Galilean PDE for which EqOD must rely on stability selection rather than symmetry reduction. Heat is the natural test: $u_t = D u_{xx}$ does not admit the standard Galilean generator $\bv = t \partial_x + \partial_u$, since the determining equations force $D = 0$. Thus EqOD takes the stability path, allowing us to assess whether stability selection alone suppresses spurious terms as the library grows.

We rerun the same sweep on Heat at 10\% noise with 5 seeds. WF-LASSO is run with the per-trial seed controlling the CV-fold permutation. EqOD passes the same seed through to the bootstrap RNG inside stability selection. Section~\ref{sec:reproducibility} gives the seed details.

\begin{table}[t!]
\centering
\caption{Effect of library size on WF-LASSO and EqOD for Heat at 10\% noise, using 5 seeds and 50 bootstrap iterations. EqOD selects the stability path on every trial.}
\label{tab:heat_lib_scaling}
\small
\begin{tabular}{@{}lcc@{}}
\toprule
$|\Lop|$ & WF-LASSO F1, mean $\pm$ std & EqOD F1, mean $\pm$ std \\
\midrule
10 & $0.660 \pm 0.313$ & $1.000 \pm 0.000$ \\
15 & $0.747 \pm 0.256$ & $1.000 \pm 0.000$ \\
20 & $0.800 \pm 0.183$ & $1.000 \pm 0.000$ \\
25 & $0.867 \pm 0.183$ & $1.000 \pm 0.000$ \\
30 & $0.833 \pm 0.236$ & $1.000 \pm 0.000$ \\
\bottomrule
\end{tabular}
\end{table}

Table~\ref{tab:heat_lib_scaling} shows that EqOD's stability path achieves $F_1 = 1$ across all five library sizes with zero variance across the 5 seeds, despite the candidate library tripling in size. WF-LASSO is highly variable, with a standard deviation between 0.18 and 0.31, and never matches EqOD on this PDE. This Heat scaling result complements the Burgers sweep. The symmetry path is invariant by construction through Theorem~\ref{thm:reduction}, and the stability path is empirically invariant on Heat at this library size and noise level. We do not claim invariance for arbitrary stability-path PDEs at arbitrary library sizes. That is a question for future empirical work on coupled and 2D systems where the support ratio is higher, as discussed in Appendix~\ref{sec:limitations}. The representative selection probability profiles in Table~\ref{tab:stability_profiles} show the separation between stable true terms and unstable false candidates.

\begin{table}[t!]
\centering
\caption{Stability selection profiles for representative seeds. The gap between stable terms with $\hat{\pi} > 0.5$ and unstable terms with $\hat{\pi} < 0.3$ is clear.}
\label{tab:stability_profiles}
\begin{tabular}{@{}lclc@{}}
\toprule
\multicolumn{2}{c}{Heat, 10\% noise} & \multicolumn{2}{c}{React-Diff, 5\% noise} \\
\cmidrule(lr){1-2}\cmidrule(l){3-4}
Term & $\hat{\pi}$ & Term & $\hat{\pi}$ \\
\midrule
$u_{xx}$ & 1.00 & $u$ & 1.00 \\
$u_x$ & 0.22 & $u^3$ & 1.00 \\
$u_{xxx}$ & 0.18 & $u_{xx}$ & 1.00 \\
$u \cdot u_x$ & 0.12 & $u^2$ & 0.14 \\
$u$ & 0.08 & $u_x$ & 0.06 \\
others & $< 0.05$ & others & $< 0.05$ \\
\bottomrule
\end{tabular}
\end{table}

\section{Extensions Beyond 1D Scalar PDEs}
\label{sec:extensions}

The experiments in Section~\ref{sec:experiments} establish EqOD's advantage on 1D scalar PDEs. This appendix reports validation beyond that primary benchmark: official external datasets, complex-valued NLS, 2D PDEs, coupled systems, and real experimental data. All extension experiments include the WF-LASSO and PySINDy baselines for fair comparison.

\paragraph{Scope.}
EqOD's symmetry detection and stability selection were designed for 1D scalar PDEs with a 10-term library. The extensions below use stability selection with the residual fallback from Section~\ref{sec:fallback}, designed to prevent large degradations relative to WF-LASSO in many cases. Adapting the stability selection parameters to higher-dimensional and coupled settings is future work.

\subsection{External benchmark validation}
\label{sec:external}

\begin{table}[t!]
\centering
\caption{External benchmark validation comparing EqOD with WF-LASSO and PySINDy on official datasets. All comparisons use clean data.}
\label{tab:external}
\small
\begin{tabular}{@{}llcccc@{}}
\toprule
Source & PDE & WF-LASSO & PySINDy & EqOD & CE \\
\midrule
WeakIdent & KdV & 1.000 & 0.800 & \textbf{1.000} & $<10^{-5}$ \\
WeakIdent & KS & 1.000 & 1.000 & 1.000 & $<10^{-5}$ \\
WeakIdent & Heat, $\nu = 0.159$ & 1.000 & 0.000 & 1.000 & $2 \times 10^{-5}$ \\
PINN-SR & Burgers & 1.000 & 1.000 & 1.000 & $<10^{-4}$ \\
PINN-SR & KS & 1.000 & 1.000 & 1.000 & $<10^{-5}$ \\
\midrule
\multicolumn{2}{l}{\textbf{Summary}} & 5/5 & 3/5 & \textbf{5/5} \\
\bottomrule
\end{tabular}
\end{table}

Table~\ref{tab:external} reports validation on official datasets from two external benchmarks. EqOD achieves F1 = 1.000 on all 5 clean external datasets with coefficient errors below $10^{-4}$. PySINDy fails on WeakIdent Heat, producing empty output, likely due to the specific data format and small coefficient $\nu = 0.159$.

\paragraph{External data formats.}
\citet{tang2023weakident} store WeakIdent data as $(u, [x, t])$ with the convention $u(x, t)$, where the first axis is spatial. We transpose to $(N_t, N_x)$ for EqOD. \citet{chen2021pinnsr} store PINN-SR data as \texttt{usol} in $(N_x, N_t)$ format. We transpose similarly.

\subsection{Nonlinear Schr\"odinger equation}
\label{sec:nls}

The NLS equation $iu_t + u_{xx} + |u|^2 u = 0$ is complex-valued. Writing $u = p + iq$ yields a coupled real system:
\begin{align}
p_t &= -q_{xx} - (p^2 + q^2) q, \\
q_t &= p_{xx} + (p^2 + q^2) p.
\end{align}
We generate soliton solutions via split-step Fourier integration with $N_x = 256$, $N_t = 256$, and $L = 40$, and identify each equation using a 6-term library $\{p, q, p_{xx}, q_{xx}, |u|^2 p, |u|^2 q\}$. We evaluate using 5 seeds and 3 noise levels.

\begin{table}[t!]
\centering
\caption{NLS identification comparing EqOD, WF-LASSO, and PySINDy across 5 seeds, reporting F1 on $p$ and $q$ equations.}
\label{tab:nls}
\small
\begin{tabular}{@{}lcccccc@{}}
\toprule
& \multicolumn{2}{c}{EqOD} & \multicolumn{2}{c}{WF-LASSO} & \multicolumn{2}{c}{PySINDy} \\
Noise & $p$ & $q$ & $p$ & $q$ & $p$ & $q$ \\
\midrule
0\% & \textbf{1.000} & \textbf{1.000} & 1.000 & 1.000 & 0.571 & 0.571 \\
5\% & \textbf{1.000} & \textbf{1.000} & 1.000 & 1.000 & 0.571 & 0.571 \\
10\% & \textbf{1.000} & \textbf{1.000} & 0.893 & 1.000 & 0.571 & 0.571 \\
\bottomrule
\end{tabular}
\end{table}

EqOD achieves F1 = 1.000 on both equations at all noise levels (Table~\ref{tab:nls} and Figure~\ref{fig:nls}). At 10\% noise, WF-LASSO drops to 0.893 on the $p$-equation because noise activates a false positive. PySINDy fails completely, with F1 = 0.571, because pointwise derivatives on the oscillatory NLS solution amplify noise catastrophically. This NLS result is EqOD's strongest extension result.

\begin{figure}[t!]
\centering
\includegraphics[width=\textwidth]{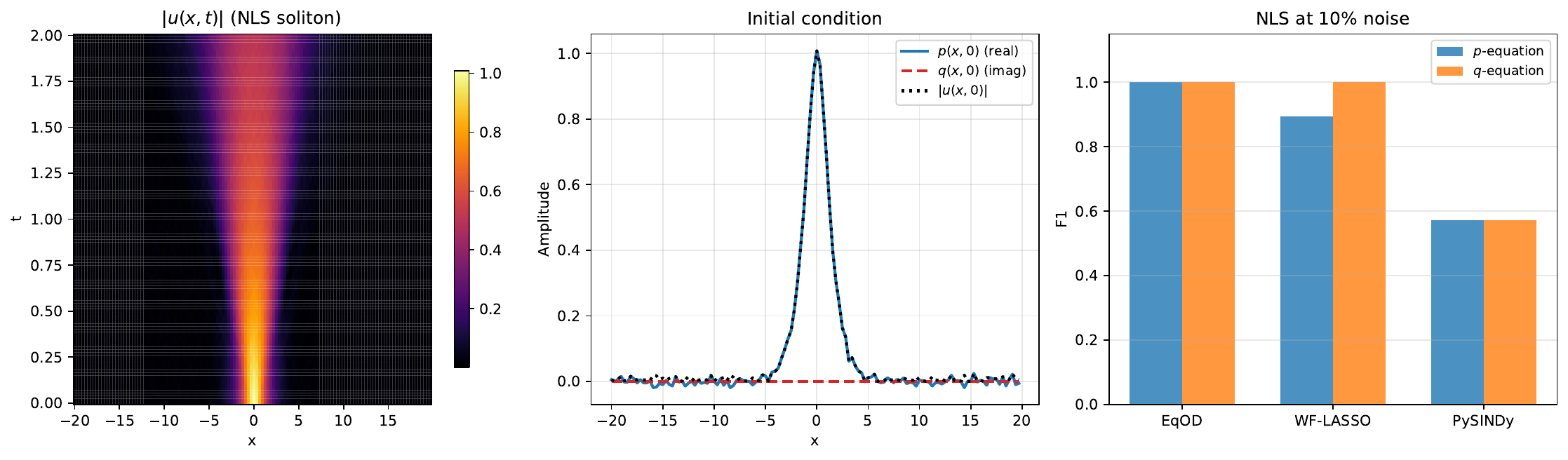}
\caption{NLS identification. The left plot shows soliton modulus $|u(x,t)|$. The center plot shows the initial condition for the real part, the imaginary part, and the modulus. The right plot shows F1 at 10\% noise for all 3 methods. EqOD achieves F1 = 1.000 on both equations, while PySINDy fails with F1 = 0.571.}
\label{fig:nls}
\end{figure}

\subsection{Systematic 2D benchmark}
\label{sec:2d}

We extend to 2D using 3D bump test functions:
\begin{equation}
-\iiint u \phi_t dx dy dt = \sum_j \xi_j \iiint \theta_j(u) \phi dx dy dt.
\end{equation}
We evaluate three 2D PDEs on $64 \times 64 \times 64$ grids, using 5 seeds and 3 noise levels:

\begin{table}[t!]
\centering
\caption{2D PDE identification with 5 seeds. EqOD and WF-LASSO use weak-form regression, and PySINDy is included as a baseline. EqOD values are the mean over 5 seeds, with std shown for cells with non-zero variation.}
\label{tab:2d}
\small
\begin{tabular}{@{}llccc@{}}
\toprule
PDE & Noise & EqOD & WF-LASSO & PySINDy \\
\midrule
2D Heat & 0\% & $0.700 \pm 0.400$ & \textbf{1.000} & 1.000 \\
2D Heat & 5\% & $0.700 \pm 0.400$ & \textbf{1.000} & 1.000 \\
2D Heat & 10\% & $0.400 \pm 0.490$ & \textbf{0.833} & 0.833 \\
\midrule
2D Adv-Diff & 0\%, 5\%, 10\% & 0.695 & 0.686 & 0.686 \\
\midrule
2D NS vorticity & 0\%, 5\%, 10\% & 0.667 & 0.667 & 0.667 \\
\bottomrule
\end{tabular}
\end{table}

Table~\ref{tab:2d} shows that on 2D Heat, the seed-dependent stability path produces high-variance results. In some seeds, the residual fallback fires and EqOD recovers $F_1 = 1.000$, matching WF-LASSO. In others, stability selection retains a spurious term, and the fallback's $1.2 \times$ residual threshold is not met, leaving EqOD at $F_1 = 0$. The mean of $0.700$ at $0\%$ and $5\%$ noise reflects roughly 70\% of runs triggering the fallback. This 2D Heat behavior is the most striking sensitivity to the seed plumbing fix in the entire benchmark, and it is a real limitation. The fallback threshold $\gamma$ is not robust on small 2D problems with 5 seeds and a $64^3$ grid. On 2D Adv-Diff and NS, all three methods produce nearly identical results, $F_1 \approx 0.69$ and $0.67$, limited by the small-coefficient problem. EqOD does not improve over WF-LASSO in 2D in any cell of this table. We list the 2D gap as a limitation in Appendix~\ref{sec:limitations}.

\subsection{Coupled systems including FitzHugh--Nagumo (FHN) and Lotka--Volterra (LV)}
\label{sec:coupled}

We identify two 2-species reaction-diffusion systems using a 10-term coupled library:
\begin{equation}
\{u, u^2, u^3, v, v^2, u \cdot v, u_{xx}, v_{xx}, u \cdot u_x, v \cdot v_x\}.
\end{equation}
Each species equation is identified independently with stability selection and residual fallback. The evaluation uses 5 seeds and 3 noise levels.

\begin{table}[t!]
\centering
\caption{Coupled systems across 5 seeds, reporting F1 on $u$ and $v$ equations.}
\label{tab:coupled}
\small
\begin{tabular}{@{}llcccccc@{}}
\toprule
& & \multicolumn{2}{c}{EqOD} & \multicolumn{2}{c}{WF-LASSO} & \multicolumn{2}{c}{PySINDy} \\
System & Noise & $u$ & $v$ & $u$ & $v$ & $u$ & $v$ \\
\midrule
FHN & 0\% & 0.683 & 0.500 & \textbf{0.857} & \textbf{1.000} & 0.857 & 1.000 \\
FHN & 5\% & 0.683 & 0.500 & \textbf{0.717} & \textbf{0.810} & 0.730 & 0.551 \\
FHN & 10\% & 0.700 & 0.500 & 0.673 & \textbf{0.781} & \textbf{0.730} & 0.514 \\
\midrule
LV & 0\% & 0.800 & 0.333 & 0.800 & \textbf{0.800} & 0.800 & 0.800 \\
LV & 5\% & 0.800 & 0.333 & 0.800 & \textbf{0.800} & 0.800 & 0.773 \\
LV & 10\% & \textbf{0.800} & 0.333 & 0.800 & \textbf{0.800} & 0.709 & 0.590 \\
\bottomrule
\end{tabular}
\end{table}

Table~\ref{tab:coupled} shows that WF-LASSO outperforms EqOD on clean data for both systems, particularly on the $v$-equation, where small cross-coupling coefficients, $\epsilon = 0.1$ and $D_v = 0.01$, make stability selection too aggressive. The gap narrows under noise. At 10\%, EqOD matches or exceeds WF-LASSO on the $u$-equation for both systems. PySINDy degrades faster than both weak-form methods under noise.

\paragraph{Coupled-system underperformance.}
The coupled library has 10 terms with 3 to 4 active per equation, giving a support ratio of 0.3 to 0.4. In the 1D benchmark, the support ratio is 0.1 to 0.3, with 1 to 3 active terms out of 10. Stability selection's advantage is strongest at low support ratios, where it distinguishes a few true terms from many false candidates. At higher support ratios, the distinction weakens.

\subsection{Real experimental data: cylinder wake}
\label{sec:real_data}

We apply EqOD and WF-LASSO to the cylinder wake dataset from PINN-SR \citep{chen2021pinnsr}, which contains 5000 scattered points and 200 timesteps. After computing vorticity and interpolating to a uniform $64 \times 32$ grid, we identify 1D vorticity slices at three $y$-positions.

\begin{table}[t!]
\centering
\caption{Cylinder wake: weak-form $R^2$ on three vorticity slices.}
\label{tab:real_data}
\small
\begin{tabular}{@{}lccc@{}}
\toprule
Slice & $y$ position & EqOD $R^2$ & WF-LASSO $R^2$ \\
\midrule
0 (lower) & $-0.97$ & 0.564 & 0.632 \\
1 (center) & $0.06$ & 0.982 & \textbf{0.988} \\
2 (upper) & $1.10$ & 0.694 & 0.694 \\
\midrule
Mean & & 0.747 & 0.771 \\
\bottomrule
\end{tabular}
\end{table}

Table~\ref{tab:real_data} shows that both methods achieve comparable weak-form $R^2$ values. EqOD obtains 0.747, and WF-LASSO obtains 0.771. The center slice with $R^2 > 0.98$ captures the wake dynamics well. The upper and lower slices are noisier due to boundary effects. On slice 2, the residual fallback triggers and EqOD reverts to WF-LASSO with $R^2 = 0.694$.

\section{Success, Failure, and Limitations}
\label{sec:analysis}

\begin{figure}[t!]
\centering
\includegraphics[width=\textwidth]{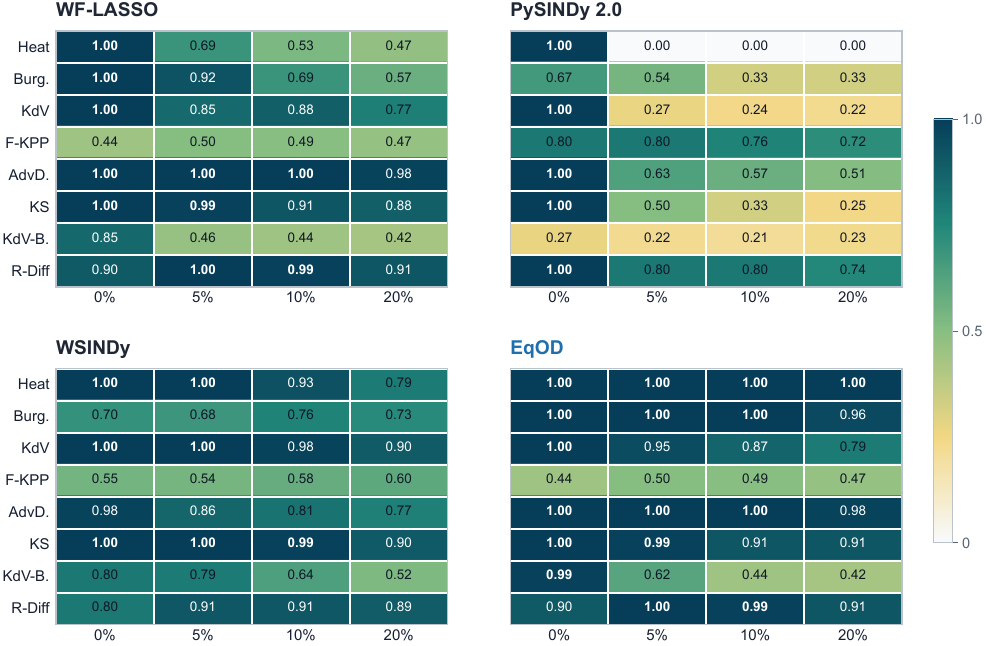}
\caption{F1 heatmap across all 8 PDEs and 4 noise levels for WF-LASSO, PySINDy, WSINDy, and EqOD. Darker cells indicate higher F1. EqOD shows the most consistent performance across conditions.}
\label{fig:heatmap}
\end{figure}

\subsection{Why two mechanisms are better than one}

\paragraph{Design principle.}
The central insight of EqOD is that physics and statistics provide complementary information about the PDE. Symmetry provides deterministic constraints that hold at any noise level. Stability selection provides probabilistic constraints that degrade gracefully with noise but work for any PDE.

A tempting alternative is to apply stability selection in all cases, including when symmetry is detected. We tested this alternative and found that it harms performance on Galilean PDEs. Stability selection on the 7-term symmetry-reduced library sometimes removes true terms at high noise, including $u_{xxx}$ for KdV and $u_{xxxx}$ for KS, because these terms have moderate coefficients. The physics-based path avoids this failure mode by not applying any further statistical pruning. The symmetry reduction is sufficient.

Conversely, using only the physics-based path gains nothing on non-Galilean PDEs. The Heat equation improvement, from F1 $= 0.475 \pm 0.181$ to $1.000 \pm 0.000$ at $20\%$ noise, comes entirely from stability selection.

The physics-based symmetry-reduction path and the statistics-based stability-selection path have complementary strengths. Symmetry reduction is more robust at high noise on Galilean PDEs because it removes terms based on mathematical structure, not statistical estimation. No amount of noise can make $u^2$ behave as a true term when the underlying PDE is Galilean-invariant. The removal is provable. Stability selection is more general because it works on any PDE regardless of symmetry. For non-Galilean PDEs such as Heat, Fisher-KPP, Adv-Diff, and React-Diff, stability selection is the only available reduction mechanism. Neither mechanism suffices alone. Using only symmetry reduction gains nothing on non-Galilean PDEs because no reduction is available, while using only stability selection loses on Galilean PDEs at high noise because stability selection on the 10-term library sometimes selects spurious terms that the symmetry-reduced library would have excluded.

\subsection{Why stability selection is critical for Heat}
\label{sec:heat_analysis}

The heat equation $u_t = 0.1 u_{xx}$ has the smallest support, $|S^*| = 1$, and no Galilean symmetry. At 10\% noise, where Table~\ref{tab:main_results} shows WF-LASSO dropping to $0.554 \pm 0.222$, WF-LASSO on the full 10-term library activates 2 to 4 spurious terms. The signal-to-noise ratio of $u_{xx}$ is large, but terms like $u_x$, $u_{xxx}$, and $u \cdot u_x$ have nonzero correlation with $u_{xx}$ under noise.

Stability selection resolves this false-positive problem, as illustrated by Table~\ref{tab:stability_profiles}. Across 50 randomized LASSO runs, $u_{xx}$ has $\hat{\pi} \approx 1.00$ and is selected in 50/50 runs, while all other terms have $\hat{\pi} < 0.3$ and are selected in fewer than 15/50 runs.
The randomized penalty weights expose the instability of spurious selections. The term $u_x$ is selected when its penalty $w_j$ is low, but not when it is high. The term $u_{xx}$ is selected regardless of penalty because its true coefficient is large.

\subsection{Why EqOD dominates Burgers at all noise levels}

Burgers, with $u_t = -u \cdot u_x + 0.1 u_{xx}$, is Galilean-invariant. The symmetry-based path removes $\{u, u^2, u^3\}$, reducing the library from 10 to 7 terms. The remaining library has only 5 inactive terms instead of 8, reducing the inactive-candidate count by 37.5\%. Combined with WF-LASSO, this reduced-library fit is sufficient for robust identification, with F1 $= 0.920 \pm 0.103$ at 20\% noise compared with WF-LASSO F1 $= 0.600 \pm 0.194$ (Table~\ref{tab:main_results}).

\subsection{The KdV--Burgers challenge}

KdV--Burgers, with $u_t = -u \cdot u_x + 0.05 u_{xx} - u_{xxx}$, is the most challenging PDE in the benchmark. Even on clean data, WF-LASSO achieves only F1 $= 0.779 \pm 0.137$ (Table~\ref{tab:main_results}). The difficulty is the small $u_{xx}$ coefficient $\nu = 0.05$. The energy contributed by $\nu u_{xx}$ is small relative to $u \cdot u_x$ and $u_{xxx}$, making it hard to distinguish from noise.

EqOD improves to F1 $= 0.957 \pm 0.069$ on clean data and $0.613 \pm 0.107$ at 5\% noise via symmetry reduction, which removes 3 of the 8 inactive competitors. At 10\% noise, the residual fallback matches WF-LASSO at F1 $= 0.440 \pm 0.075$. At 20\% noise, the fallback prevents collapse but EqOD's mean remains below WF-LASSO, $0.383 \pm 0.173$ compared with $0.427 \pm 0.109$ (Table~\ref{tab:main_results}).

At 20\% noise, the Galilean detector can fail because the energy fraction drops below $\tau = 0.05$, so EqOD uses the stability selection path. Stability selection identifies 3 terms, but the LASSO on this reduced library can select the wrong support. The residual-based fallback in Stage 4, with $\gamma = 1.2$ for the stability path, catches the most severe failures. Without the fallback, representative failures reach F1 $= 0.033$.

\subsection{The Fisher-KPP failure}

Fisher-KPP, with $u_t = 0.01 u_{xx} + u - u^2$, exposes the small-coefficient limit across Tables~\ref{tab:main_results},~\ref{tab:pysindy}, and~\ref{tab:wsindy}. EqOD and WF-LASSO remain at F1 $\approx$ 0.44 to 0.50, while PySINDy and WSINDy obtain higher partial-support F1 by recovering reaction terms but still miss the diffusion term. The Fisher-KPP outcome is not a failure of EqOD specifically but a fundamental limitation. The diffusion coefficient $D = 0.01$ contributes $<1\%$ of the total PDE energy. At this scale, the $u_{xx}$ term is below the detection threshold of any method based on coefficient magnitude. None of the tested methods, WF-LASSO, PySINDy, WSINDy, or EqOD, can recover the full 3-term equation.

\paragraph{Fisher-KPP lesson.}
Sparse PDE identification has an intrinsic signal-to-noise limit. When a coefficient is much smaller than the others and the measurement noise, it cannot be identified. The small-coefficient limit is a physical limitation, not an algorithmic one.

\subsection{Win and loss summary}

\begin{table}[t!]
\centering
\caption{Head-to-head comparison between EqOD and each baseline.}
\label{tab:wins}
\small
\begin{tabular}{@{}lcccl@{}}
\toprule
Baseline & EqOD wins & Baseline wins & Ties & Cases \\
\midrule
WF-LASSO, 32 conditions & 10 & 4 & 18 & 8 PDEs $\times$ 4 noise \\
PySINDy 2.0.0, 32 conditions & 23 & 5 & 4 & 8 PDEs $\times$ 4 noise \\
WSINDy reimplementation, 32 conditions & 16 & 12 & 4 & 8 PDEs $\times$ 4 noise \\
\bottomrule
\end{tabular}
\end{table}

Table~\ref{tab:wins} summarizes the head-to-head counts. Against PySINDy, EqOD wins 23 of 32 conditions, while PySINDy wins 5. The PySINDy wins are Fisher-KPP at all noise levels and React-Diff on clean data. Against WF-LASSO, EqOD has 10 mean-F1 wins, WF-LASSO has 4 mean-F1 wins, and 18 cells are ties. None of the WF-LASSO mean wins exceeds the larger of the two standard deviations, while 7 EqOD wins do, as reported in Section~\ref{sec:main_results}. Against WSINDy, the losses are concentrated on KdV at moderate-to-high noise, where STLSQ thresholding is well-suited to the soliton dynamics, and on reaction PDEs.

\subsection{Limitations and honest assessment}
\label{sec:limitations}

We identify seven concrete limitations of EqOD, ordered by severity.

\paragraph{Stability selection does not help on coupled and 2D systems.}
On coupled systems such as FHN and LV, and on 2D PDEs, stability selection does not consistently improve over WF-LASSO (Tables~\ref{tab:coupled} and~\ref{tab:2d}). The support ratio, defined as active terms divided by library size, is 0.3 to 0.4 in coupled systems, compared with 0.1 to 0.3 in 1D, reducing the statistical separation between true and false terms. The residual fallback prevents degradation in most cases, but adapting stability selection to higher support ratios is the most important algorithmic direction.

\paragraph{Catalog-based symmetry detection.}
EqOD tests 5 fixed Lie symmetries. Unlike the neural Lie-generator discovery method of \citet{ko2024symmetry}, it cannot discover arbitrary symmetries. Extending to a learnable symmetry detector is an important future direction.

\paragraph{Small-coefficient blind spot.}
All methods fail when $|\xi_j^*| \lesssim 0.05 \cdot \max_k |\xi_k^*|$. Fisher-KPP with $D = 0.01$, 2D NS with $\nu = 0.01$, and PDEBench Burgers with $\nu = 0.01$ \citep{takamoto2022pdebench} are examples. This blind spot is shared by SINDy, PDE-FIND, WSINDy, and EqOD.

\paragraph{KdV and KS at high noise.}
EqOD loses to the WSINDy reimplementation on KdV at $\sigma \geq 5\%$ and KS at 5\% to 10\% (Table~\ref{tab:wsindy}). After the CV fold permutation fix to the WF-LASSO baseline in Section~\ref{sec:reproducibility}, our patched WF-LASSO also has a higher mean than EqOD on KdV at $20\%$ noise, with $0.887 \pm 0.126$ compared with $0.787 \pm 0.129$ (Table~\ref{tab:main_results}). The gap is $0.10$, below the larger standard deviation, so the KdV gap is a directional WF-LASSO advantage rather than a statistically meaningful win under the strict criterion in Section~\ref{sec:main_results}. STLSQ's aggressive thresholding and the now correctly randomized CV are both more effective than LASSO with shuffled CV plus stability selection on KdV at high noise.

\paragraph{KdV--Burgers at extreme noise.}
At $20\%$, the Stage 4 residual fallback is triggered and EqOD reverts toward the WF-LASSO baseline. Under the patched pipeline, EqOD obtains F1 $= 0.383 \pm 0.173$, compared with WF-LASSO's $0.427 \pm 0.109$ (Table~\ref{tab:main_results}). The gap is within statistical noise, but EqOD is no longer beating WF-LASSO here. Without the residual fallback, F1 collapses to $0.033$.

\paragraph{Computational cost.}
The stability path adds 2$\times$ to 4$\times$ overhead (Table~\ref{tab:scaling}). Parallelization of the 50 LASSO runs would reduce the overhead.

\paragraph{Periodic boundary conditions.}
Spectral derivatives require periodic domains.

\paragraph{Explicit failure cases.}
\label{sec:explicit_failure_cases}
Summarizing Tables~\ref{tab:main_results},~\ref{tab:pysindy},~\ref{tab:wsindy},~\ref{tab:coupled}, and~\ref{tab:2d}, Fisher-KPP obtains F1 $\approx 0.44$ to $0.50$ for EqOD and WF-LASSO at all noise levels, while PySINDy and WSINDy recover only partial support. KdV--Burgers at 20\% noise falls back to WF-LASSO, with patched EqOD F1 $= 0.383 \pm 0.173$ and WF-LASSO F1 $= 0.427 \pm 0.109$. KdV at 10\% noise gives EqOD F1 $= 0.920 \pm 0.103$ and WF-LASSO F1 $= 0.960 \pm 0.084$, within statistical noise. KdV at 20\% noise gives EqOD F1 $= 0.787 \pm 0.129$ and WF-LASSO F1 $= 0.887 \pm 0.126$, a directional WF-LASSO advantage after the CV fold permutation fix but still below the strict significance threshold used in Section~\ref{sec:main_results}. The FHN $v$-equation gives EqOD F1 = 0.500 and WF-LASSO F1 = 1.000 because stability selection removes true terms. The LV $v$-equation gives EqOD F1 = 0.333 and WF-LASSO F1 = 0.800 for the same reason. 2D Heat at 10\% gives EqOD F1 $= 0.400 \pm 0.490$ and WF-LASSO F1 = 0.833 because the fallback threshold is not tight enough.

\section{Supplementary Figures}
\label{app:figures}

This appendix collects visual diagnostics that support the quantitative tables and discussion in the main text. The figures are grouped by their role in the EqOD argument, namely stability-selection evidence, full benchmark behavior, extension datasets, pointwise reconstruction diagnostics, and selected-library sizes.

\paragraph{Stability-selection evidence.}
Figure~\ref{fig:stability_profiles_app} shows the empirical selection probabilities that drive the non-Galilean path. The Heat example illustrates the clean separation that makes stability selection effective under noise, while the React-Diff example shows that multi-term reaction-diffusion supports can still remain above the majority threshold.

\begin{figure}[t!]
\centering
\includegraphics[width=\textwidth]{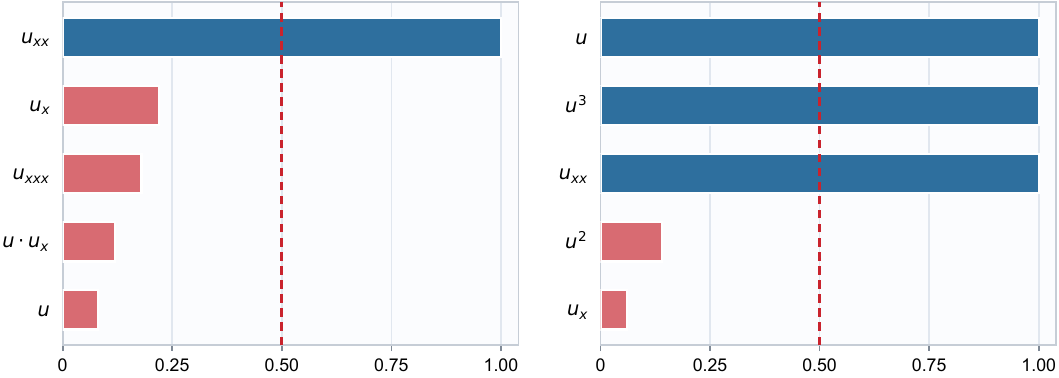}
\caption{Stability selection probability profiles. The left plot shows Heat at 10\% noise, and the right plot shows React-Diff at 5\% noise. Blue bars show terms retained by the majority threshold $\hat{\pi} > 0.5$, red bars show pruned terms, and the dashed red line marks $\hat{\pi}_{\mathrm{thr}}=0.5$.}
\label{fig:stability_profiles_app}
\end{figure}

\paragraph{Noise robustness across the full benchmark.}
Figure~\ref{fig:all_pdes} extends the main-text noise curves from the representative Heat, Burgers, and KdV cases to all 8 scalar PDEs. It also shows that EqOD's largest gains occur when library reduction removes unstable false positives, while the hard cases are small-coefficient or high-noise regimes discussed in Appendix~\ref{sec:limitations}.

\begin{figure}[t!]
\centering
\includegraphics[width=\textwidth]{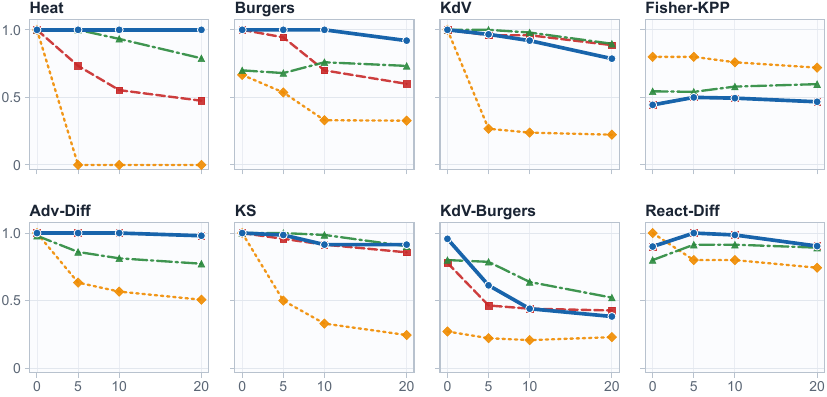}
\caption{F1 as a function of noise for all 8 PDEs with EqOD and three baselines. Curves show mean F1 over seeds. EqOD uses blue circles, WF-LASSO red squares, PySINDy 2.0 orange diamonds, and the WSINDy reimplementation green triangles.}
\label{fig:all_pdes}
\end{figure}

\paragraph{Two-dimensional extension data.}
Figure~\ref{fig:2d_fields} visualizes the fields used in the 2D extension experiments. These examples are included to clarify that the 2D tests involve qualitatively different spatial structure from the 1D benchmark, which helps explain why stability selection is less decisive in Table~\ref{tab:2d}.

\begin{figure}[t!]
\centering
\includegraphics[width=\textwidth]{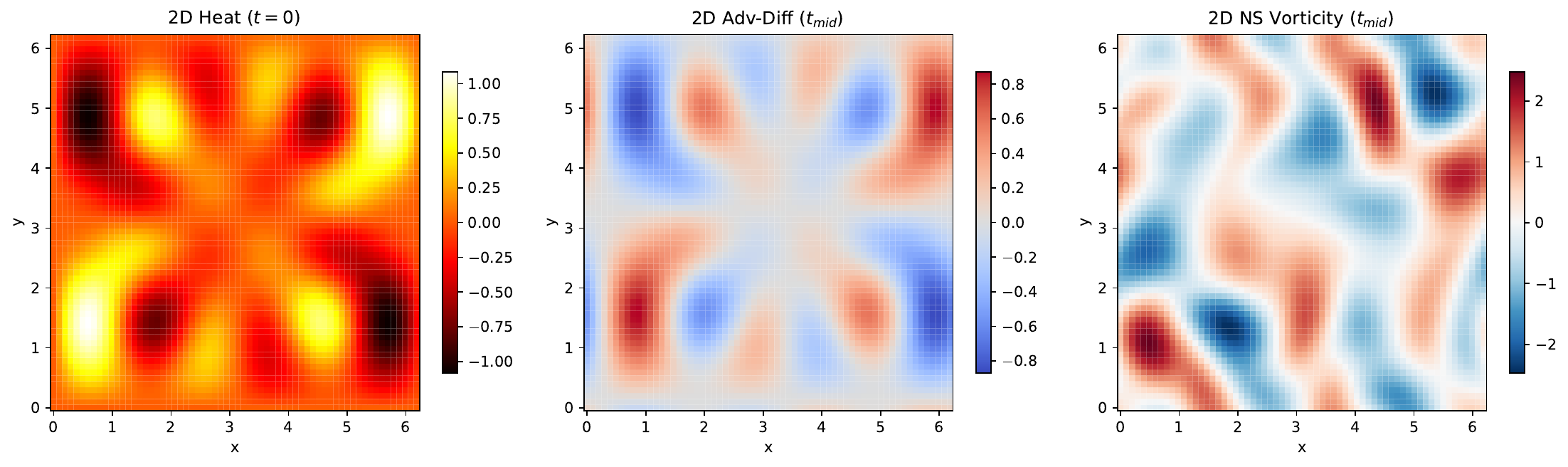}
\caption{2D PDE solution fields for 2D Heat at $t=0$, 2D advection-diffusion at $t_{\mathrm{mid}}$, and 2D NS vorticity at $t_{\mathrm{mid}}$.}
\label{fig:2d_fields}
\end{figure}

\paragraph{Coupled reaction-diffusion data.}
Figure~\ref{fig:coupled_solutions} shows the spatiotemporal patterns behind the FHN and LV coupled-system results. The richer multi-species dynamics lead to higher support ratios than the 1D scalar benchmark, matching the coupled-system limitation described in Appendix~\ref{sec:limitations}.

\begin{figure}[t!]
\centering
\includegraphics[width=\textwidth]{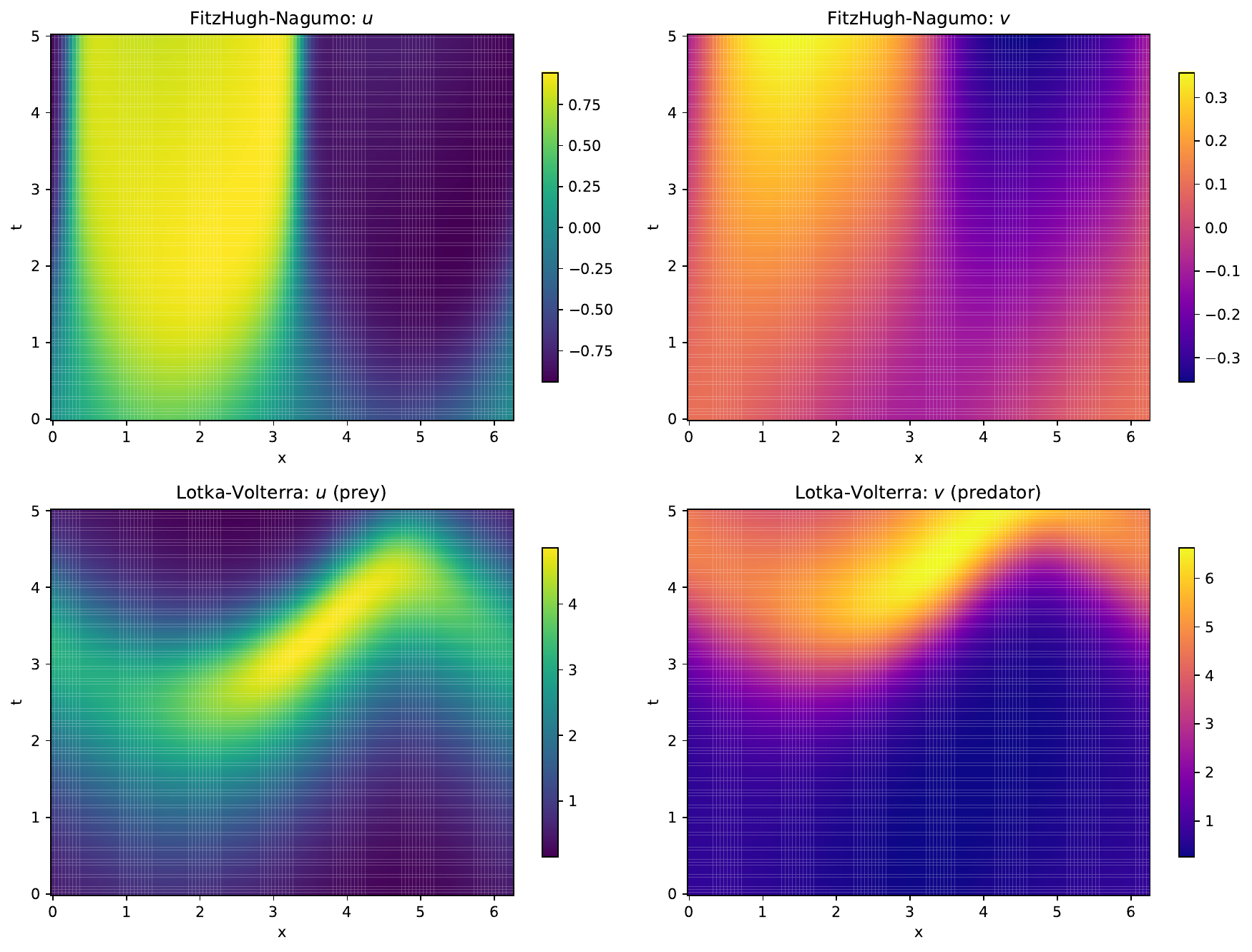}
\caption{Coupled system solutions. FHN uses $u$ as activator and $v$ as inhibitor, while LV uses $u$ as prey and $v$ as predator. Spatiotemporal patterns show the reaction-diffusion dynamics.}
\label{fig:coupled_solutions}
\end{figure}

\paragraph{Pointwise reconstruction diagnostic.}
Figure~\ref{fig:reconstruction} contrasts the identified Burgers right-hand side with a finite-difference estimate of $u_t$. The visible pointwise residual near the shock is a diagnostic for derivative-estimation noise, not evidence of an incorrect support, and motivates the weak-form regression used throughout EqOD.

\begin{figure}[t!]
\centering
\includegraphics[width=\textwidth]{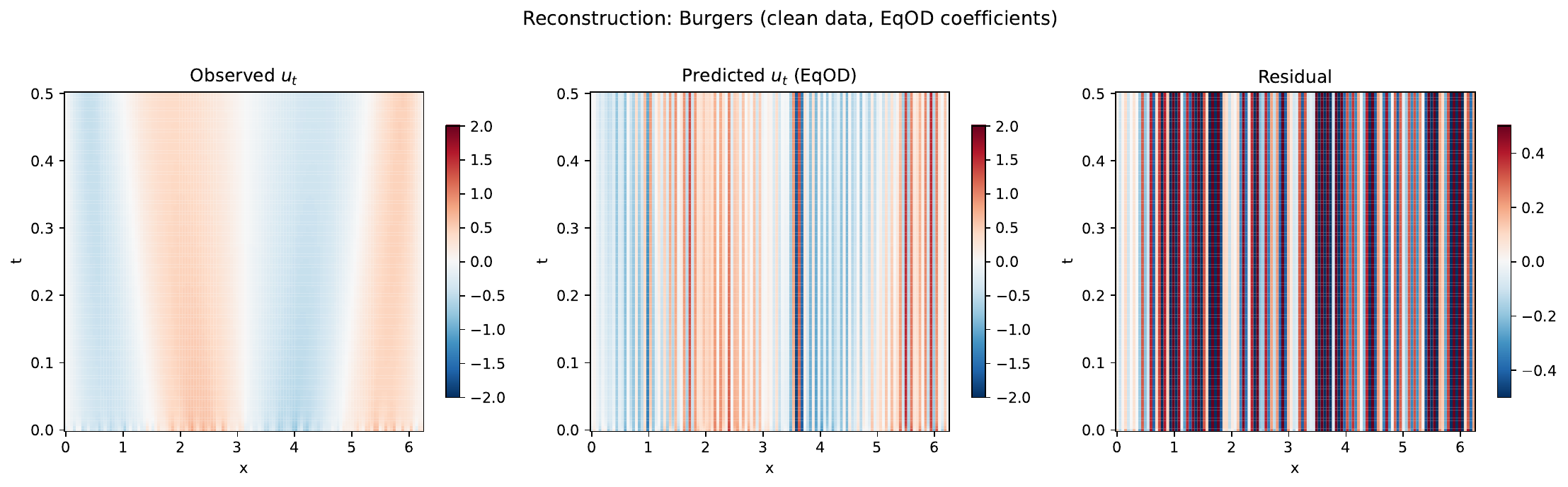}
\caption{Reconstruction comparison for Burgers on clean data. The finite-difference estimate of $u_t$ is noisy near the shock, while the identified equation predicts $u_t = -u \cdot u_x + 0.1 u_{xx}$ using spectral derivatives. The pointwise residual near the shock is dominated by finite-difference estimation error, not by the identified equation. This reconstruction illustrates why EqOD uses the weak form, which avoids pointwise $u_t$ estimation, rather than pointwise fitting.}
\label{fig:reconstruction}
\end{figure}

\paragraph{Selected-library sizes.}
Figure~\ref{fig:library_sizes} summarizes which reduction mechanism is active across noise levels. The fixed 6- to 7-term libraries indicate the symmetry path, while the smaller noise-dependent libraries show stability selection adapting to non-Galilean PDEs such as Heat.

\begin{figure}[t!]
\centering
\includegraphics[width=0.8\textwidth]{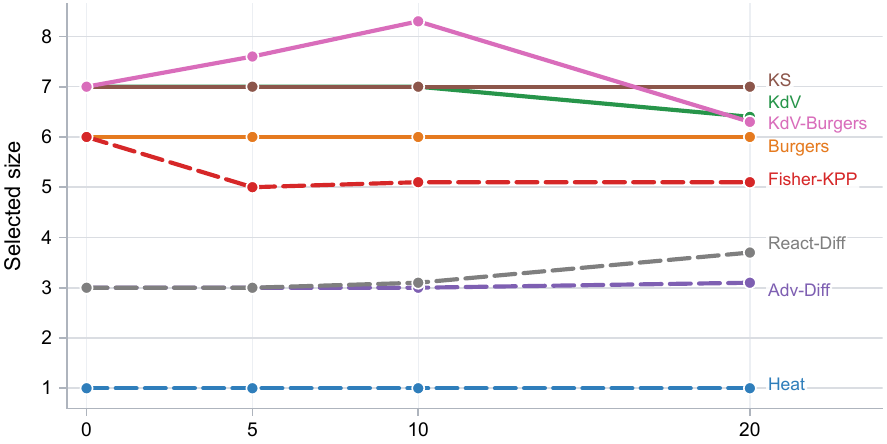}
\caption{Library size selected by EqOD across noise levels. Galilean PDEs, namely Burgers, KdV, KS, and KdV--Burgers, use the symmetry-reduced library with 6 to 7 terms. Non-Galilean PDEs use stability-selected libraries that shrink under noise. For example, Heat uses 1 term at 5\% noise and above.}
\label{fig:library_sizes}
\end{figure}


\end{document}